\documentclass[letterpaper]{article} % DO NOT CHANGE THIS
\usepackage{aaai23}  % DO NOT CHANGE THIS
\usepackage{times}  % DO NOT CHANGE THIS
\usepackage{helvet}  % DO NOT CHANGE THIS
\usepackage{courier}  % DO NOT CHANGE THIS
\usepackage[hyphens]{url}  % DO NOT CHANGE THIS
\usepackage{graphicx} % DO NOT CHANGE THIS
\usepackage{natbib}  % DO NOT CHANGE THIS AND DO NOT ADD ANY OPTIONS TO IT
\usepackage{caption} % DO NOT CHANGE THIS AND DO NOT ADD ANY OPTIONS TO IT
\usepackage{algorithm}
\usepackage{algorithmic}

\usepackage{newfloat}
\usepackage{listings}
\DeclareCaptionStyle{ruled}{labelfont=normalfont,labelsep=colon,strut=off} % DO NOT CHANGE THIS
\floatstyle{ruled}
\newfloat{listing}{tb}{lst}{}
\floatname{listing}{Listing}
\title{AP: Selective Activation for De-sparsifying Pruned Neural Networks}

\author {
	Shiyu Liu,
	Rohan Ghosh, 
	Dylan Tan,
	Mehul Motani
}
\affiliations {
	Department of Electrical and Computer Engineering \\
	School of Engineering, National University of Singapore \\
	shiyu\_liu@u.nus.edu, rghosh92@gmail.com, eledtht@nus.edu.sg, motani@nus.edu.sg
}

\usepackage{graphicx}
\usepackage{tikz}
\usepackage{xcolor}
\usepackage{comment}
\usepackage{amsthm, amssymb} % define this before the line numbering.
\usepackage{color}

\theoremstyle{definition}
\newtheorem{definition}{Definition}
\newtheorem{theorem}{Theorem}
\newtheorem{corollary}{Corollary}
\newtheorem{remark}{Remark}

\usepackage[accsupp]{axessibility}  % Improves PDF readability for those with disabilities.

\usepackage{scalerel}
\usepackage{algorithm}
\usepackage{algorithmic}
\usepackage{xcolor}

\usepackage{array}
\usepackage{eqparbox}

\usepackage{etoolbox}  % patch def of algorithmic environment
\makeatletter
\patchcmd{\algorithmic}{\addtolength{\ALC@tlm}{\leftmargin} }{\addtolength{\ALC@tlm}{\leftmargin}}{}{}
\makeatother

\usepackage{graphicx}
\graphicspath{{img/}}
\usepackage{caption}
\usepackage{booktabs}
\usepackage{enumerate}

\usepackage{tikz}
\usetikzlibrary{patterns}

\newlength\brickwidth
\newlength\brickheight
\newlength\tickstep
\colorlet{slashcolor}{blue!50}
\colorlet{horizcolor}{red!50}

\pgfdeclarepatternformonly{horizontal tight lines}{\pgfpointorigin}{\pgfqpoint{100pt}{0.5pt}}{\pgfqpoint{100pt}{1.5pt}}%
{%
	\pgfsetlinewidth{0.6pt}%
	\pgfpathmoveto{\pgfqpoint{0pt}{0.5pt}}%
	\pgfpathlineto{\pgfqpoint{100pt}{0.5pt}}%
	\pgfusepath{stroke}%
}%

\usepackage{enumitem}

\usepackage{breqn}
\usepackage{balance}

\usepackage{breqn}

\begin{document}

\maketitle

\begin{abstract}
The rectified linear unit (ReLU) is a highly successful activation function in neural networks as it allows networks to easily obtain sparse representations, which reduces overfitting in overparameterized networks. However, in network pruning, we find that the sparsity introduced by ReLU, which we quantify by a term called dynamic dead neuron rate (DNR), is not beneficial for the pruned network. Interestingly, the more the network is pruned, the smaller the dynamic DNR becomes during optimization. 
This motivates us to propose a method to explicitly reduce the dynamic DNR for the pruned network, i.e., de-sparsify the network. We refer to our method as Activating-while-Pruning (AP). We note that AP does not function as a stand-alone method, as it does not evaluate the importance of weights. Instead, it works in tandem with existing pruning methods and aims to improve their performance by selective activation of nodes to reduce the dynamic DNR. We conduct extensive experiments using popular networks (e.g., ResNet, VGG) via two classical and three state-of-the-art pruning methods. The experimental results on public datasets (e.g., CIFAR-10/100) suggest that AP works well with existing pruning methods and improves the performance by 3\% - 4\%. For larger scale datasets (e.g., ImageNet) and state-of-the-art networks (e.g., vision transformer), we observe an improvement of 2\% - 3\% with AP as opposed to without. Lastly, we conduct an ablation study to examine the effectiveness of the components comprising AP.  
\end{abstract}

\section{Introduction}
The rectified linear unit (ReLU) \cite{glorot2011deep}, $\sigma (x)$ = $\max$\{x, 0\}, is the most widely used activation function in neural networks (e.g., ResNet \cite{resnet18}, Transformer \cite{vaswani2017attention}). The success of ReLU is mainly due to fact that existing networks tend to be overparameterized and ReLU can easily regularize overparameterized networks by introducing sparsity (i.e., post-activation output is zero) \cite{glorot2011deep}, leading to  
promising results in many computer vision tasks (e.g., image classification \cite{vgg16,resnet18}, object detection \cite{dai2021dynamic,joseph2021towards}).

In this paper, we study the ReLU's sparsity constraint in the context of network pruning (i.e., a method of compression that removes weights from the network). Specifically, we question the utility of ReLU's sparsity constraint, when the network is no longer overparameterized during iterative pruning. In the following, we summarize the workflow of our study together with our contributions.

\begin{enumerate}[topsep=0pt]
	\setlength{\itemsep}{0pt}
	\item {\bf Motivation and Theoretical Study. }In Section \ref{sec3.1}, we introduce a term called dynamic Dead Neuron Rate (DNR), which quantifies the sparsity introduced by ReLU neurons that are not completely pruned during iterative pruning. Through rigorous experiments on popular networks (e.g., ResNet \cite{resnet18}), we find that the more the network is pruned, the smaller the dynamic DNR becomes during optimization. This suggests that the sparsity introduced by ReLU is not beneficial for pruned networks. Further theoretical investigations also reveal the importance of reducing dynamic DNR for pruned networks from an information bottleneck (IB) \cite{ib_tishby_dl} perspective (see Sec. \ref{sec3.2}). 
	
	\item {\bf A Method for De-sparsifying Pruned Networks.} In Section \ref{sec3.3}, we propose a method called Activating-while-Pruning (AP) which aims to explicitly reduce dynamic DNR. We note that AP does not function as a stand-alone method, as it does not evaluate the importance of weights. Instead, it works in tandem with existing pruning methods and aims to improve their performance by reducing dynamic DNR. The proposed AP has two variants: (i) AP-Lite which slightly improves the performance of existing methods, but without increasing the algorithm complexity, and (ii) AP-Pro which introduces an addition retraining step to the existing methods in every pruning cycle, but significantly improves the performance of existing methods.
	
	\item {\bf Experiments. }In Section \ref{sec4}, we conduct experiments on CIFAR-10/100 \cite{cifar10} with several popular networks (e.g., ResNet, VGG) using two classical and three state-of-the-art (SOTA) pruning methods. The results demonstrate that AP works well with existing pruning methods and improve their performance by 3\% - 4\%. For the larger scale dataset (e.g., ImageNet \cite{deng2009imagenet}) and SOTA networks (e.g., vision transformer \cite{dosovitskiy2020image}), we observe an improvement of 2\% - 3\% with AP as opposed to without.
	
	%AP also demonstrates promising results, with an improvement of 3 - 5\% over existing methods.
	
	\item {\bf Ablation Study. }In Section \ref{sec4.3}, we carry out an ablation study to investigate and demonstrate the effectiveness of several key components that make up the proposed AP.
\end{enumerate}

\section{Background}
\label{sec2}
\vspace{-1mm}

Network pruning is a method used to reduce the size of the neural network, with its first work \cite{lecun1998gradient} dating back to 1990. In terms of the pruning style, all existing methods can be divided into two classes: (i) {\bf One-Shot Pruning} and (ii) {\bf Iterative Pruning}. Assuming that we plan to prune $Q\%$ of the parameters of a trained network, a typical {\bf pruning cycle} consists of three basic steps:

\begin{enumerate}[topsep=0pt]
	%\vspace{-0.6pc}
	%\setlength{\itemsep}{0.5pt}
	\item Prune $\eta \%$ of existing parameters based on given metrics. 
	
	\item Freeze pruned weights as zero.
	
	\item Retrain the pruned network to recover the performance.
	%\vspace{-2mm}
\end{enumerate}

In One-Shot Pruning, $\eta$ is set to $Q$ and the parameters are pruned in one pruning cycle. While for Iterative Pruning, a much smaller portion of parameters (i.e., $\eta<<Q$) are pruned per pruning cycle. The pruning process is repeated multiple times until $Q\%$ of parameters are pruned. As for performance, Iterative Pruning often results in better performance compared to One-Shot Pruning \cite{han2015learning,frankle2018lottery,li2016pruning}. So far, existing works aim to improve the pruning performance by exploring either new pruning metrics or new retraining methods.

{\bf Pruning Metrics. } Weight magnitude is the most popular approximation metric used to determine less useful connections; the intuition being that smaller magnitude weights have a smaller effect in the output, and hence are less likely to have an impact on the model outcome if pruned \cite{he2020learning,li2020eagleeye,li2020weight}. Many works have investigated the use of weight magnitude as the pruning metric, i.e. \cite{han2015learning,frankle2018lottery}. More recently, \cite{lee2020layer} introduced layer-adaptive magnitude-based pruning (LAMP) and attempts to prune weights based on a scaled version of the magnitude. \cite{park2020lookahead} proposed a method called Lookahead Pruning (LAP), which evaluates the importance of weights based on the impact of pruning on neighbor layers. Another popular metric used for pruning is via the gradient; the intuition being that weights with smaller gradients are less impactful in optimizing the loss function. Examples are \cite{lecun1998gradient,theis2018faster}, where \cite{lecun1998gradient} proposed using the second derivative of the loss function with respect to the parameters (i.e., the Hessian Matrix) as a pruning metric and \cite{theis2018faster} used Fisher information to approximate the Hessian Matrix. A recent work \cite{Blalock2020} reviewed numerous pruning methods and suggested two classical pruning methods for performance evaluation: %(i) {\bf Global Magnitude}: Pruning the weights with the lowest absolute value anywhere in the network and (ii) {\bf Global Gradient}: Pruning the weights with the lowest absolute value of (weight $\times$ gradient) anywhere in the network.
\begin{enumerate}[noitemsep, topsep=0pt]
	\item {\bf Global Magnitude}: Pruning weights with the lowest absolute value anywhere in the network.
	\item {\bf Global Gradient}: Pruning weights with the lowest absolute value of (weight$\times$gradient) anywhere in the network.
\end{enumerate}

{\bf Retraining Methods. } Another factor that significantly affects the pruning performance is the retraining method. According to \cite{han2015learning}, Han et al. trained the unpruned network with a learning rate schedule and retrained the pruned network using a constant learning rate (i.e., often the final learning rate of the learning rate schedule). A recent work \cite{renda2020comparing} proposed learning rate rewinding which used the same learning rate schedule to retrain the pruned network, leading to a better pruning performance. More recently, \cite{S-Cyc} attempted to optimize the choice of learning rate (LR) during retraining and proposed a LR schedule called S-Cyc. They showed that S-Cyc could work well with various pruning methods, further improving the existing performance. Most notably, \cite{frankle2018lottery} found that resetting the unpruned weights to their original values (known as {\bf weight rewinding}) after each pruning cycle could lead to even higher performance than the original model. Some follow-on works \cite{zhou2019deconstructing,renda2020comparing,malach2020proving} investigated this phenomenon more precisely and applied this method in other fields (e.g., transfer learning \cite{mehta2019sparse}, reinforcement learning and natural language processing \cite{yu2019playing}). 

{\bf Other Works. } In addition to works mentioned above, several other works also share some deeper insights on network pruning \cite{liu2018rethinking,zhu2017prune,liu2018darts,wang2020pruning}. For example, \cite{liu2018darts} demonstrated that training from scratch on the right sparse architecture yields better results than pruning from pre-trained models. Similarly, \cite{wang2020pruning} suggested that the fully-trained network could reduce the search space for the pruned structure. More recently, \cite{luo2020neural} addressed the issue of pruning residual connections with limited data and \cite{ye2020good} theoretically proved the existence of small subnetworks with lower loss than the unpruned network.

\section{Activating-while-Pruning}
\label{sec3}
In Section \ref{sec3.1}, we first conduct experiments to evaluate the DNR during iterative pruning. Next, in Section \ref{sec3.2}, we link the experimental results to theoretical studies and motivate Activating-while-Pruning (AP). In Section \ref{sec3.3}, we introduce the idea of AP and present its algorithm. Lastly, in Section \ref{sec3.4}, we illustrate how AP can improve on the performance of existing pruning methods.

\begin{figure*}[!t]
	
	\hspace{0mm}\begin{minipage}{0.5\textwidth}
		\begin{center}
		\includegraphics[width=0.95\linewidth]{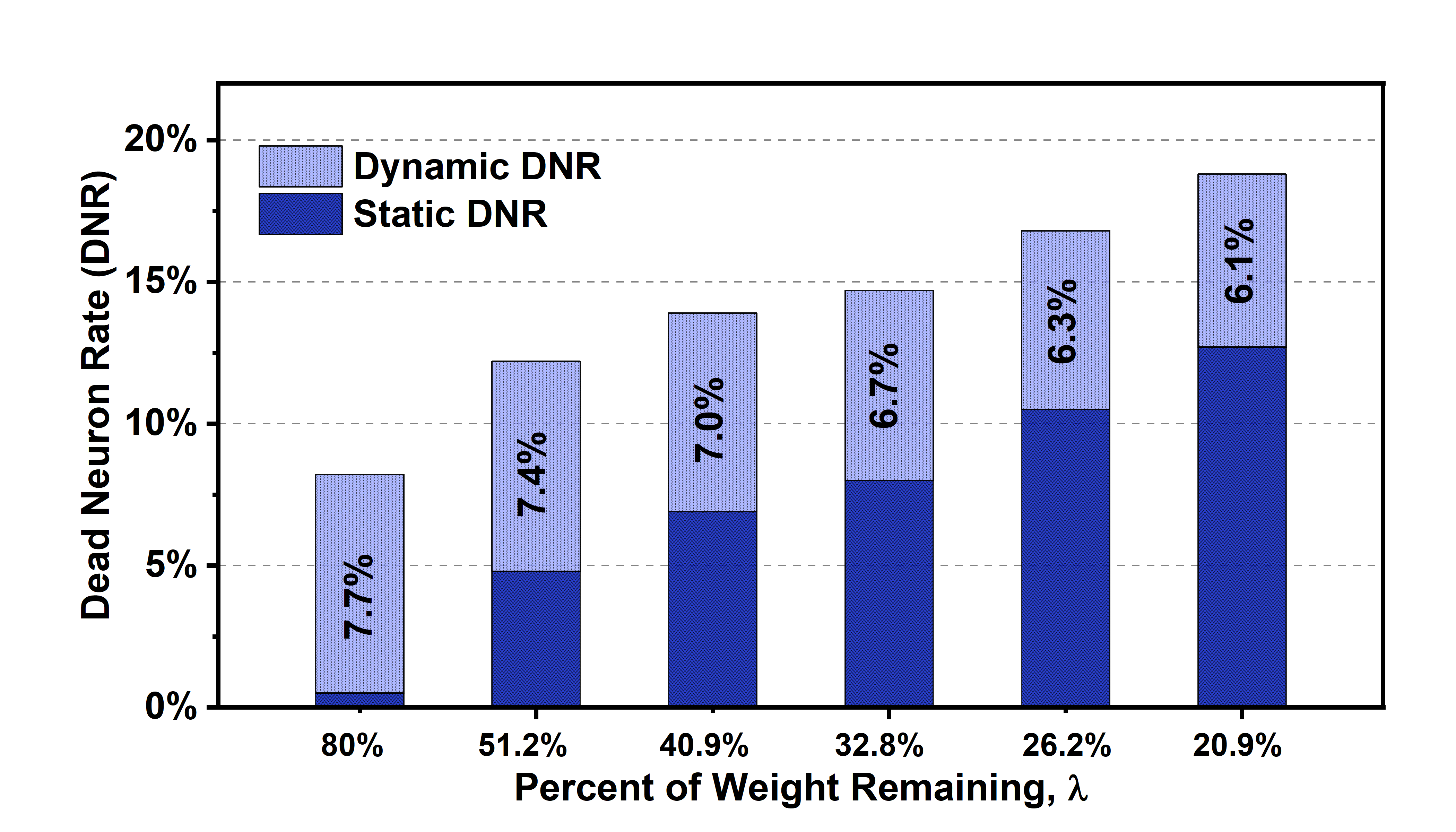}
		\end{center}
	\end{minipage}%
	\hspace{0mm}\begin{minipage}{0.5\textwidth}
		\begin{center}
		\includegraphics[width=0.95\linewidth]{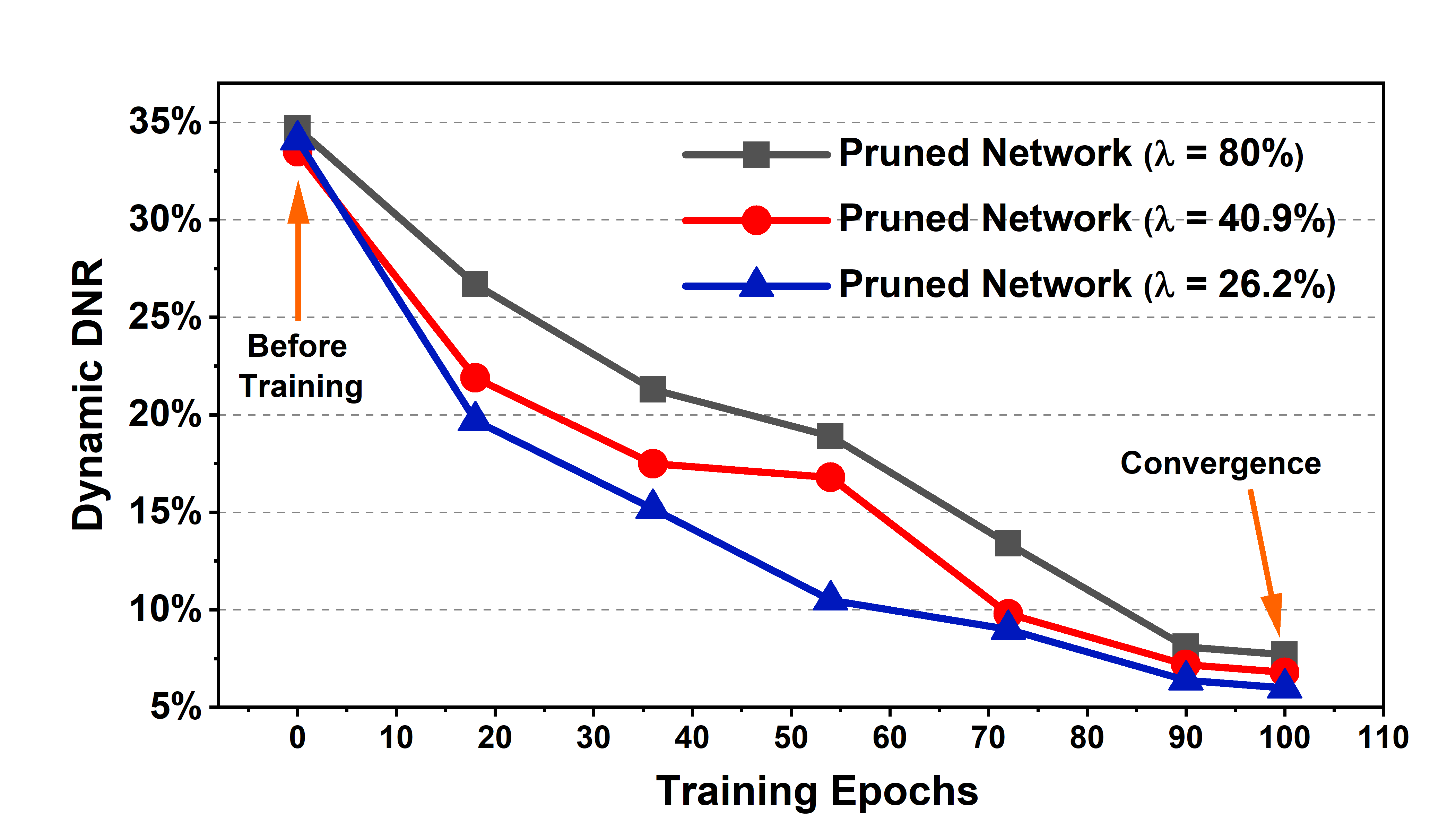}
		\end{center}
	\end{minipage}
    \vspace{-4mm}
	\caption{Dynamic and static Dead Neuron Rate (DNR) when iterative pruning ResNet-20 on CIFAR-10 using Global Magnitude. Left: dynamic and static DNR when the network converges; Right: dynamic DNR during optimization.} 
	\label{fig_1}
	\vspace{-3mm}
\end{figure*}

\subsection{Experiments on DNR}
\label{sec3.1}
We study the state of the ReLU function during iterative pruning and introduce a term called Dead Neuron Rate (DNR), which is the percentage of dead ReLU neurons (i.e., a neuron with a post-ReLU output of zero) in the network averaged over all training samples when the network converges. Mathematically, the DNR can be written as
\begin{equation}
	\resizebox{.43\textwidth}{!} 
	{$\text{DNR} = \frac{1}{n} \sum_{i = 1}^{n}\frac{\text{\# of dead ReLU neurons}}{\text{all neurons in the unpruned network}},$
	}
\end{equation}
where $n$ is the number of training samples. We classify a dead neuron as either dynamically dead or statically dead. The {\bf dynamically dead neuron} is a dead neuron in which not all of the weights have been pruned. Hence, it is not likely to be permanently dead and its state depends on its input. As an example, a neuron can be dead for a sample X, but it could be active (i.e., post-ReLU output $>$ 0) for a sample Y. The DNR contributed by dynamically dead neurons is referred to as {\bf dynamic DNR}. The {\bf statically dead neuron} is a dead neuron in which all associated weights have been pruned. The DNR contributed by statically dead neurons is referred to as {\bf static DNR}.

DNR is a term that we introduce to quantify the sparsity introduced by ReLU. Many similar sparsity metrics have been proposed in the literature \cite{hurley2009comparing}. As an example, the Gini Index \cite{goswami2016sparsity} computed from Lorenz curve (i.e., plot the cumulative percentages of total quantities) can be used to evaluate the sparsity of network graphs. Another popular metric will be Hoyer measure \cite{hoyer2004non} which is the ratio between L1 and L2 norms, can also be used to evaluate the sparsity of networks. The closest metric to DNR is parameter sparsity \cite{goodfellow2016deep} which computes the percentage of zero-magnitude parameters among all parameters. We note that both parameter sparsity and DNR will contribute to sparse representations, and in this paper, we use DNR to quantify the sparsity introduced by ReLU.

{\bf Experiment Setup and Observations. }Given the definition of DNR, static and dynamic DNR, we conduct pruning experiments using ResNet-20 on the CIFAR-10 dataset with the aim of examining the benefit (or lack thereof) of ReLU's sparsity for pruned networks. We iteratively prune ResNet-20 with a pruning rate of 20 (i.e., 20\% of existing weights are pruned) using the Global Magnitude (i.e., prune weights with the smallest magnitude anywhere in the network). We refer to the standard implementation reported in \cite{renda2020comparing,frankle2018lottery} (i.e., SGD optimizer \cite{ruder2016overview}, 100 training epochs and batch size of 128, learning rate warmup to 0.03 and drop by a factor of 10 at 55 and 70 epochs) and compute the static DNR and dynamic DNR while the network is iteratively pruned. The experimental results are shown in Fig. \ref{fig_1}, where we make two observations.

\begin{enumerate}[noitemsep, topsep=0pt]
	\item As shown in Fig. \ref{fig_1} (left), the value of DNR (i.e., sum of static and dynamic DNR) increases as the network is iteratively pruned. As expected, static DNR grows as more weights are pruned during iterative pruning.
	
	\item Surprisingly, dynamic DNR tends to decrease as the network is iteratively pruned (see Fig. \ref{fig_1} (left)), suggesting that pruned networks do not favor the sparsity of ReLU. In Fig. \ref{fig_1} (right), we show that for pruned networks with different $\lambda$ (i.e., percent of remaining weights), they have similar dynamic DNR before training, but the pruned network with smaller $\lambda$ tends to have a smaller dynamic DNR during and after optimization.
\end{enumerate}

{\bf Result Analysis. }One possible reason for the decrease in dynamic DNR could be due to the fact that once the neuron is dead, its gradient becomes zero, meaning that it stops learning and degrades the learning ability of the network \cite{lu2019dying,Stork2020}. This could be beneficial as existing networks tend to be overparameterized and dynamic DNR may help to reduce the occurrence of overfitting. However, for pruned networks whose learning ability are heavily degraded, the dynamic DNR could be harmful as a dead ReLU always outputs the same value (zero as it happens) for any given non-positive input, meaning that it takes no role in discriminating between inputs. Therefore, during retraining, the pruned network attempts to restore its performance by reducing its dynamic DNR so that the extracted information can be passed to the subsequent layers. Similar performance trends can be observed using VGG-19 with Global Gradient (see Fig. \ref{fig_2} in the {\bf Appendix}). Next, we present a {\bf theoretical study} of DNR and show its relevance to the network's ability to discriminate. 

\subsection{Relevance to IB and Complexity}
\label{sec3.2}
Here, we present a theoretical result and subsequent insights that highlight the relevance of the dynamic DNR of a certain layer of the pruned network, to the Information Bottleneck (IB) Principle proposed in \cite{ib_tishby_dl}. In the IB setting, the computational flow is usually denoted by the graph $X\xrightarrow{}T\xrightarrow{}Y$, where $X$ represents the input, $T$ represents the extracted representation, and $Y$ represents the network's output. The IB principle states that the goal of training networks must be to minimize the mutual information \cite{cover2006elements} between $X$ and $T$ (denoted as $I(X;T)$) while keeping $I(Y;T)$ large. Overly compressed features (low $I(X;T)$) will not retain enough information to predict the labels, whereas under-compressed features (high $I(X;T)$) imply that more label-irrelevant information is retained in $T$ which can adversely affect generalization. Next, we provide a few definitions. 
% With this, we now outline our first theoretical result which highlights the relevance of $D_{{\scaleto{DNR}{3pt}}}(T)$ and $S_{{\scaleto{DNR}{3pt}}}(T)$ to $I(X;T)$, as follows. 
% Our objective here is to demonstrate how increase of dynamic DNR of any particular feature layer denoted by $T$, can negatively affect the information that the network can channel through $T$, which in turn can hamper the network's ability to fit the training data labels.

\begin{definition}
	\textbf{Layer-Specific dynamic DNR ($D_{{\scaleto{DNR}{3pt}}}(T)$):} We are given a dataset $S=\{X_1,...,X_m\}$, where $X_i\sim P$ $\forall i$ and $P$ is the data generating distribution. We denote the dynamic DNR of only the neurons at a certain layer within the network, represented by the vector $T$, by $D_{{\scaleto{DNR}{3pt}}}(T)$. $D_{{\scaleto{DNR}{3pt}}}(T)$ is computed over the entire distribution of input in $P$. 
\end{definition}

% We also define the sample-specific extension of $DDNR(T)$, denoted as $\widehat{DDNR}(T)$, which only represents the dynamic DNR of the neurons in $T$ over the given dataset. 
\begin{definition}
	\textbf{Layer-Specific static DNR ($S_{{\scaleto{DNR}{3pt}}}(T)$):} It is defined in the same manner as $D_{{\scaleto{DNR}{3pt}}}(T)$, instead of computing the dynamic DNR, we compute the static DNR of $T$. 
\end{definition}
With this, we now outline our first theoretical result which highlights the relevance of $D_{{\scaleto{DNR}{3pt}}}(T)$ and $S_{{\scaleto{DNR}{3pt}}}(T)$ to $I(X;T)$, as follows. 
\vspace{-1mm}
% The sample-specific extension $\widehat{SDNR}(T)$ is also similarly defined.
\begin{theorem} \label{thm:1}
	We are given the computational flow $X\xrightarrow{}T\xrightarrow{}Y$, where $T$ represents the features at some arbitrary depth within a network, represented with finite precision (e.g. float32 or float64). We only consider the subset of network configurations for which (a) the activations in $T$ are less than a threshold $\tau$ and (b) the zero-activation probability of each neuron in $T$ is upper bounded by some $p_S<1$. Let $dim(T)$ represent the dimensionality of $T$, i.e., the number of neurons at that depth. We then have,
    
    \begin{dmath}
        \begin{split}
		I(X;T)\leq & C\times dim(T)\times
		\Big(1-S_{{\scaleto{DNR}{3pt}}}(T) \\
		& - D_{{\scaleto{DNR}{3pt}}}(T) \Big(1 -  \frac{1}{C}\log{\frac{1-S_{{\scaleto{DNR}{3pt}}}(T)}{D_{{\scaleto{DNR}{3pt}}}(T)}} \Big)\Big), \nonumber
		\end{split}
	\end{dmath}
	
	for a finite constant $C$ that only depends on the network architecture, $\tau$ and $p_S$.
\end{theorem}

The following corollary addresses the dependencies of Theorem \ref{thm:1}. The proof of Theorem \ref{thm:1} and Corollary \ref{corr:1.1} are provided in the Appendix.
\begin{corollary}\label{corr:1.1}
	The upper bound for $I(X;T)$ in Theorem \ref{thm:1} decreases in response to increase of $D_{{\scaleto{DNR}{3pt}}}(T)$ and $S_{{\scaleto{DNR}{3pt}}}(T)$. 
\end{corollary}

\vspace{-1mm}
\begin{remark}
	\textbf{(Relevance to Complexity)} We see that \cite{ib_gen} notes how the metric $I(X;T)$ represents the \textit{effective complexity} of the network. As Theorem 3 in \cite{ib_gen} shows, $I(X;T)$ captures the dependencies between $X$ and $T$ and directly correlates with the network's function fitting ability. Coupled with the observations from Theorem \ref{thm:1} and Corollary \ref{corr:1.1}, for a fixed pruned network configuration (i.e., fixed $S_{{\scaleto{DNR}{3pt}}}(T)$), greater $D_{{\scaleto{DNR}{3pt}}}(T)$ will likely reduce the \textit{effective complexity} of the network, undermining the learning ability of the neural network.
\end{remark}
\vspace{-2mm}
\begin{remark}
	\textbf{(Motivation for AP)} Theorem \ref{thm:1} also shows
	that a pruned network, which possesses large $S_{{\scaleto{DNR}{3pt}}}(T)$, leads to a higher risk of \textit{over-compression} of information (low $I(X;T)$). To address this issue, we can reduce the dynamic DNR (from Corollary \ref{corr:1.1}) so that the upper bound of $I(X;T)$ can be increased, 
	mitigating the issue of \textit{over-compression} for a pruned network.
	This agrees with our initial motivation that the sparsity introduced by ReLU is not beneficial for the pruned network and reducing dynamic DNR helps to improve the learning ability of the pruned network. 
\end{remark}

\subsection{Algorithm of Activating-while-Pruning}
\label{sec3.3}

The experimental and theoretical results above suggest that, in order to better preserve the learning ability of pruned networks, a smaller dynamic DNR is preferred. This motivates us to propose Activating-while-Pruning (AP) which aims to explicitly reduce dynamic DNR.

We note that the proposed AP does not work alone, as it does not evaluate the importance of weights. Instead, it serves as a booster to existing pruning methods and help to improve their pruning performance by reducing dynamic DNR (see Fig. \ref{fig_3}). Assume that the pruning method X removes $p$\% of weights in every pruning cycle (see the upper part in Fig. \ref{fig_3}). After using AP, the overall pruning rate remains unchanged as $p$\%, but $(p-q)$\% of weights are pruned according to the pruning method X with the aim of pruning less important weights, while $q$\% of weights are pruned according to AP (see the lower part in Fig. \ref{fig_3}) with the aim of reducing dynamic DNR. Consider a network $f(\theta)$ with ReLU activation function. Two key steps to reducing dynamic DNR are summarized as follows.

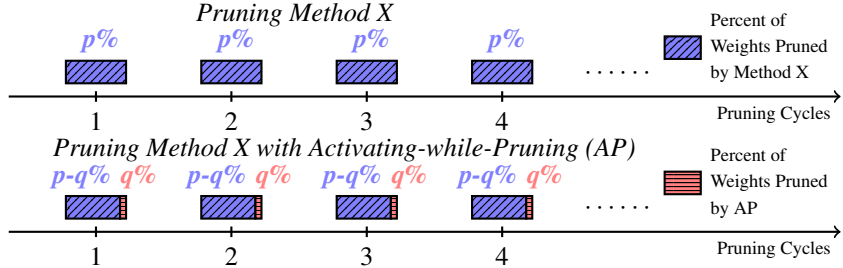
\begin{figure*}[!t]
	\centering
	\begin{tikzpicture}[thick]
		
		\node[anchor=south west] at (3.0,.8) {\it Pruning Method X};
		
		\node[anchor=south west, text = slashcolor] at (1.5,.45) {\it \textbf{p\%}};
		\node[anchor=south west, text = slashcolor] at (3.3,.45) {\it \textbf{p\%}};
		\node[anchor=south west, text = slashcolor] at (5.1,.45) {\it \textbf{p\%}};
		\node[anchor=south west, text = slashcolor] at (6.9,.45) {\it \textbf{p\%}};
		
		\draw[shift={(5.2\tickstep,0.5)}, pattern=north east lines, preaction={fill=slashcolor}] 
		(0,0) rectangle (.6\brickwidth, \brickheight) 
		node [align=flush left, yshift=-.5\brickheight, right, font=\scriptsize] 
		{Percent of \\ Weights Pruned \\ by Method X};
		
		\draw[->] (0.8\brickwidth, 0) -- (6.5\tickstep, 0) node [below left, align=flush right, font=\scriptsize] {Pruning Cycles};
		
		\foreach \i in {1,...,4} {
			\begin{scope}[xshift=\i\tickstep]
				\draw (0, 2pt) -- (0, -2pt) node [below] {\i};
				\path[draw, pattern=north east lines, preaction={fill=slashcolor}] 
				(-.5\brickwidth,5pt) rectangle +(\brickwidth,\brickheight);
			\end{scope}
		};
		\node[inner sep=0pt, anchor=south west] at (4.6\tickstep, 2pt+.5\brickheight) {$\cdots\cdots$};

		\tikzset{yshift=-3.0cm}
		\node[anchor=south west] at (1.1, 2.0) {\it Pruning Method X with Activating-while-Pruning (AP)};
		
		\tikzset{yshift=0.2cm}
		\node[anchor=south west, text = slashcolor] at (1.1, 1.45) {\it \textbf{p-q\%}};
		\node[anchor=south west, text = horizcolor] at (2.0, 1.45) {\it \textbf{q\%}};
		
		\node[anchor=south west, text = slashcolor] at (2.9, 1.45) {\it \textbf{p-q\%}};
		\node[anchor=south west, text = horizcolor] at (3.8, 1.45) {\it \textbf{q\%}};
		
		\node[anchor=south west, text = slashcolor] at (4.7, 1.45) {\it \textbf{p-q\%}};
		\node[anchor=south west, text = horizcolor] at (5.6, 1.45) {\it \textbf{q\%}};
		
		\node[anchor=south west, text = slashcolor] at (6.5, 1.45) {\it \textbf{p-q\%}};
		\node[anchor=south west, text = horizcolor] at (7.4, 1.45) {\it \textbf{q\%}};
		
		\draw[shift={(5.2\tickstep,1.5)}, pattern=horizontal tight lines, preaction={fill=horizcolor}]
		(0,0) rectangle (.6\brickwidth, \brickheight) 
		node [align=flush left, yshift=-.5\brickheight, right, font=\scriptsize] 
		{Percent of \\ Weights Pruned \\ by AP};
		
		\draw[->] (0.8\brickwidth, 1) -- (6.5\tickstep, 1) node [below left, align=flush right, font=\scriptsize] {Pruning Cycles};
		
		\foreach \i/\j in {1/0.9, 2/0.9, 3/0.9, 4/0.9} {
			\begin{scope}[xshift=\i\tickstep, yshift = 1cm]
				\draw (0, 2pt) -- (0, -2pt) node [below] {\i};
				\path[draw, pattern=north east lines, preaction={fill=slashcolor}] 
				(-.5\brickwidth, 5pt) rectangle +(\j\brickwidth,\brickheight);
				\path[draw, pattern=horizontal tight lines, preaction={fill=horizcolor}]
				(-.5\brickwidth, 5pt)+(\j\brickwidth,\brickheight) rectangle +(\brickwidth, 0pt);
			\end{scope}
		};
		\node[inner sep=0pt, anchor=south west, yshift = 1cm] at (4.6\tickstep, 2pt+.5\brickheight) {$\cdots\cdots$};
		
	\end{tikzpicture}
	\vspace*{-4mm}
	\caption{Illustration of how AP works in tandem with existing pruning methods (e.g., method X) in each pruning cycle.}
	\label{fig_3}
	\vspace{-3  mm}
\end{figure*}

{\bf (1) Locate Dead ReLU Neurons. } Consider a neuron in the hidden layer with ReLU activation function, taking $n$ inputs \{${X_1W_1, ..., X_nW_n} | X_i \in \mathbb{R}$ is the input and $W_i \in \mathbb{R}$ is the associated weight\}. Let $j$ be the pre-activated output of the neuron (i.e., $j = \sum_{i = 1}^{n} X_iW_i$) and $\mathcal{J}$ be the post-activated output of the neuron ($\mathcal{J} = ReLU (j)$). Let $\mathcal{L}$ be the loss function and assume the neuron is dead ($\mathcal{J} = 0$), then the gradient of its associated weights (e.g., $W_1$) with respected to the loss function will be $\frac{\partial \mathcal{L}}{\partial W_1} = \frac{\partial \mathcal{L}}{\partial \mathcal{J}} \cdot \frac{\partial \mathcal{J}}{\partial j}\cdot \frac{\partial j}{\partial W_1} = 0$ as $\frac{\partial \mathcal{J}}{\partial j} = 0$. If a neuron is often dead during training, the weight movement of its associated weights is likely to be smaller than other neurons. Therefore, we compute the difference between weights at initialization ($\theta_0$) and the weights when the network convergences ($\theta_*$), i.e., $|\theta_* - \theta_0|$ and use it as a heuristic to locate dead ReLU neurons.

{\bf (2) Activate Dead ReLU Neurons.} Assume we have located a dead neuron in the hidden layer with $n$ inputs \{${X_1W_1, ..., X_nW_n}$ $| X_i \in \mathbb{R}$ is the input and $W_i \in \mathbb{R}$ is the associated weight\}. We note that $X_i$ is non-negative as $X_i$ is usually the post-activated output from the previous layer (i.e., the output of ReLU is non-negative). Therefore, a straightforward way to activate the dead neuron is to prune the weights with the negative value. By pruning such negative weights, we can increase the value of the pre-activation output, which may turn the the pre-activation output into positive so as to reduce dynamic DNR.

\subsection{How AP Improves Existing Methods}
\label{sec3.4}

We now summarize how AP can improve existing pruning methods in Algorithm \ref{algorithm2}, where the upper part is the algorithm of a standard iterative pruning method called pruning method X and the lower part is the algorithm of method X with AP. The proposed AP has two variants: {\bf AP-Pro} and {\bf AP-Lite}. We note that both AP-Pro and AP-Lite contain the same three steps, summarized as follows. 

\begin{enumerate}[noitemsep, topsep=0pt]
	%\vspace{-0.8pc}
	%\setlength{\itemsep}{3pt}
	\item {\bf Pruning.} Given a network at convergence with a set of dynamically dead ReLU neurons, $\mathcal{N}_1 = \{n_1, n_2, n_3,...\}$. The pruning step of AP aims to activate these dynamically dead ReLU neurons (i.e., reduce dynamic DNR) so as to preserve the learning ability of the pruned network (see the pruning metric of AP in algorithm \ref{algorithm1}).
	
	\item {\bf Weight Rewinding.} Resetting unpruned weights to their values at the initialization. We note that different weight initializations could lead to different sets of $\mathcal{N}$. In the previous step, AP aims to reduce dynamic DNR for the target $\mathcal{N}_1$ and weight rewinding attempts to prevent the target $\mathcal{N}_1$ from changing too much. Since the weights of ReLU neurons in $\mathcal{N}_1$ have been pruned by AP, these neurons could become active during retraining. The effect of weight rewinding is evaluated via an ablation study.
	
	\item {\bf Retraining} the pruned network to recover performance.
	%\vspace{-2mm}
\end{enumerate}

The key difference between AP-Lite and AP-Pro is that AP-Lite applies these three steps only once at the end of pruning (i.e., when all pruning cycles ends). It aims to slightly improve the performance, but does not substantially increase the algorithm complexity. For AP-Pro, it applies the three steps above in every pruning cycle, which increases the algorithm complexity (mainly due to the retraining step), but aims to significantly improve the performance, which could be preferred in performance oriented tasks.

\begin{algorithm}[!t]
	\setlength{\textfloatsep}{0pt}
	\algsetup{linenosize=\footnotesize}
	\footnotesize
	\caption{\footnotesize The Pruning Metric of the Proposed AP} 
	\begin{algorithmic}[1]
		\REQUIRE (i) Network $f$ with unpruned weights $\theta_0$ at initialization, $f(\theta_0)$; (ii) Network $f$ with unpruned weights $\theta_*$ at convergence,  $f(\theta_*)$; (iii) Pruning Rate of AP, $q$; \\
		\STATE {\bf Locate Dead Neurons}: Sort $\vert \theta_*$ - $\theta_0 \vert $ in an ascending order.
		\STATE {\bf Activate Dead Neurons}: In the ascending order of $\vert \theta_*$ - $\theta_0 \vert $, prune first $q\%$ of weights with the negative magnitude.
	\end{algorithmic}
	\label{algorithm1}
\end{algorithm}

\setlength{\textfloatsep}{6pt}
\begin{algorithm}[!t]
	\algsetup{linenosize=\footnotesize}
	\footnotesize
	\caption{\footnotesize The Pruning Method X with and without AP } 
	\label{algorithm2}
	\begin{algorithmic}[1]
		\REQUIRE (i) Network, $f(\theta)$; (ii) Pruning Rate of Method X, $p$; (iii) Pruning Rate of AP, $q$; (iv) Pro\_Flag = \{0: \textcolor{blue}{AP-Lite}, 1: \textcolor{red}{AP-Pro}\}; (v) Pruning Cycles, $n$; \\
		------------------------ {\bf Pruning Method X} --------------------------
		\FOR {$i = 1$ to $n$}
		\STATE Randomly initialize unpruned weights, $\theta \leftarrow \theta_0$.
		\STATE Train the network to convergence, arriving at parameters $\theta_*$.
		\STATE Prune $p$ \% of $\theta_*$ according to the pruning method X.
		\ENDFOR
		\STATE Retrain: Retrain the network to recover its performance. \\
		------------ {\bf Pruning Method X with Proposed AP} ------------    
		\FOR {$i = 1$ to $n$}
		\STATE Randomly initialize unpruned weights, $\theta \leftarrow \theta_0$.
		\STATE Train the network to convergence, arriving at parameters $\theta_*$.
		\STATE Prune ($p$ - $q$) \% of  $\theta_*$ according to the pruning method X.
		\IF [\textcolor{red}{Execution of AP-Pro}] {Pro\_Flag}  
		\STATE (i) \textcolor{red}{\it Pruning: Prune $q$ \% of parameter $\theta_*$ according to the metric of AP (see details in Algorithm \ref{algorithm1}).} 
		\STATE (ii) \textcolor{red}{\it Weight Rewinding: Reset the remaining parameters to their values in $\theta_0$.}
		\STATE (iii) \textcolor{red}{\it Retrain: Retrain the pruned network to recover its performance.}
		\ENDIF
		\ENDFOR
		\IF [\textcolor{blue}{Execution of AP-Lite}] {NOT Pro\_Flag}
		\STATE (i) \textcolor{blue}{\it Pruning: Prune $q$ \% of the parameters  $\theta_*$ according to the metric of AP (see details in Algorithm \ref{algorithm1}).} 
		\STATE (ii) \textcolor{blue}{\it Weight Rewinding: Reset the remaining parameters to their values in $\theta_0$.}
		\STATE (iii) \textcolor{blue}{\it Retrain: Retrain the pruned network to recover its performance.}
		\ENDIF
	\end{algorithmic}
\end{algorithm}

\vspace{-2mm}
\section{Performance Evaluation}
\label{sec4}

%In this section, we evaluate the performance of two classical pruning methods and three SOTA pruning methods, with and without AP. We first summarize the experiment setup in Section \ref{sec4.1}. Next, in Section \ref{sec4.2}, we compare and analyze the results obtained. Lastly, in Section \ref{sec4.3}, we conduct an ablation study to evaluate the effectiveness of several components in AP.

\subsection{Experiment Setup}
\label{sec4.1}

{\bf (1) Experiment Details.} To demonstrate that AP can work well with different pruning methods, we shortlist two classical and three state-of-the-art pruning methods. The details for each experiment are summarized as follows.

\begin{enumerate}
%[noitemsep, topsep=0pt]
	%\vspace{-0.6pc}
	%\setlength{\itemsep}{0pt}
	
	\item Pruning ResNet-20 on the CIFAR-10 dataset using Global Magnitude with and without AP. 
	
	\item Pruning VGG-19 on the CIFAR-10 dataset using Global Gradient with and without AP.
	
	\item Pruning DenseNet-40 \cite{huang2017densely} on CIFAR-100 using Layer-Adaptive Magnitude-based Pruning (LAMP) \cite{lee2020layer} with and without AP.
	
	\item Pruning MobileNetV2 \cite{sandler2018mobilenetv2} on the CIFAR-100 dataset using Lookahead Pruning (LAP) \cite{park2020lookahead} with and without AP.
	
	\item Pruning ResNet-50 \cite{resnet18} on the ImageNet (i.e., ImageNet-1000) using Iterative Magnitude Pruning (IMP) \cite{frankle2018lottery} with and without AP.
	
	\item Pruning Vision Transformer (ViT-B-16) on CIFAR-10 using IMP with and without AP. 
\end{enumerate}

%In each run, the dataset is randomly split into three parts: training dataset (60\%), validation dataset (20\%) and testing dataset (20\%). 

We train the network using SGD with momentum = 0.9 and a weight decay of 1$\texttt{e}$-4 (same as \cite{renda2020comparing,frankle2018lottery}). For the benchmark pruning method, we prune the network with a pruning rate $p$ = 20 (i.e., 20\% of existing weights are pruned) in 1 pruning cycle. After using AP, the overall pruning rate remains unchanged as 20\%, but 2\% of existing weights are pruned based on AP, while the other 18\% of existing weights are pruned based on the benchmark pruning method to be compared with (see Algorithm \ref{algorithm2}). We repeat 25 pruning cycles in 1 run and use the early-stop top-1 test accuracy (i.e., the corresponding test accuracy when early stopping criteria for validation error is met) to evaluate the performance. The experimental results averaged over 5 runs and the corresponding standard deviation are summarized in Tables \ref{per_1} - \ref{per_6}, where $\lambda$ is the percentage of weights remaining.

\begin{table}[!t]
	\centering
	\renewcommand{\arraystretch}{1.0}
	\setlength\tabcolsep{6.5pt}
	\begin{tabular}{|c|cccc|}
		\toprule
		\toprule
		\multicolumn{5}{|c|}{Original Top-1 Test Accuracy: 91.7\% ($\lambda$ = 100\%)} \\ \toprule
		{\footnotesize $\lambda$} & {\footnotesize 32.8\%} & {\footnotesize 26.2\%} & {\footnotesize 13.4\%} & {\footnotesize 5.72\%} \\ \toprule
		Glob Mag  & 90.3$\pm${\scriptsize 0.4} & 89.8$\pm${\scriptsize 0.6} & 88.2$\pm${\scriptsize 0.7} & 81.2$\pm${\scriptsize 1.1} \\
		AP-Lite & 90.4$\pm${\scriptsize 0.7} & 90.2$\pm${\scriptsize 0.8} & 88.7$\pm${\scriptsize 0.7} & 82.4$\pm${\scriptsize 1.4} \\
		AP-Pro & \textbf{90.7$\pm${\scriptsize 0.6}} & \textbf{90.4$\pm${\scriptsize 0.4}} & \textbf{89.3$\pm${\scriptsize 0.8}} & \textbf{84.1$\pm${\scriptsize 1.1}} \\ 
		\bottomrule
		\bottomrule
	\end{tabular}
	\vspace{-2.3mm}
	\caption{Performance (top-1 test accuracy $\pm$ standard deviation) of pruning ResNet-20 on CIFAR-10 using Global Magnitude (Glob Mag) with and without AP.}
	\label{per_1}
	\vspace{-0.3mm}
\end{table}

\begin{table}[!t]
	\centering
	\renewcommand{\arraystretch}{1.0}
	\setlength\tabcolsep{6.5pt}
	\begin{tabular}{|c|cccc|}
		\toprule
		\toprule
		\multicolumn{5}{|c|}{Original Top-1 Test Accuracy: 92.2\% ($\lambda$ = 100\%)} \\ \toprule
		{\footnotesize $\lambda$} & {\footnotesize 32.8\%} & {\footnotesize 26.2\%} & {\footnotesize 13.4\%} & {\footnotesize 5.72\%} \\ \toprule
		Glob Grad & 90.2$\pm${\scriptsize 0.5}  & 89.8$\pm${\scriptsize 0.8} & 89.2$\pm${\scriptsize 0.8} & 76.9$\pm${\scriptsize 1.1} \\
		AP-Lite & 90.5$\pm${\scriptsize 0.8} & 90.3$\pm${\scriptsize 0.7} & 89.7$\pm${\scriptsize 0.9} & 78.4$\pm${\scriptsize 1.4} \\
		AP-Pro & \textbf{90.8$\pm${\scriptsize 0.6}} & \textbf{90.7$\pm${\scriptsize 0.9}}& \textbf{90.4$\pm${\scriptsize 0.8}} & \textbf{79.2$\pm${\scriptsize 1.3}}  \\ 
		\bottomrule
		\bottomrule
	\end{tabular}
	\vspace{-2.3mm}
	\caption{Performance (top-1 test accuracy $\pm$ standard deviation) of pruning VGG-19 on CIFAR-10 using Global Gradient (Glob Grad) with and without AP. }
	\label{per_2}
\end{table}

{\bf (2) Hyper-parameter Selection and Tuning. }To ensure fair comparison against prior results, we utilize standard implementations (i.e., network hyper-parameters and learning rate schedules) reported in the literature. Specifically, the implementations for Tables \ref{per_1} - \ref{per_6} are from  \cite{frankle2018lottery}, \cite{zhao2019variational},  \cite{chin2020towards}, \cite{renda2020comparing} and \cite{dosovitskiy2020image}. The implementation details can be found in Section \ref{A2_new} of the {\bf Appendix}. In addition, we also tune hyper-parameters using the validation dataset via grid search. Some hyper-parameters are tuned as follows. (i) The training batch size is tuned from \{64, 128, ...., 1024\}. (ii) The learning rate is tuned from 1e-3 to 1e-1 via a stepsize of 2e-3. (iii) The number training epochs is tuned from 80 to 500 with a stepsize of 20. The validation performance using our tuned parameters are close to that of using standard implementations. Therefore, we use standard implementations reported in the literature to reproduce benchmark results.

{\bf (3) Reproducing Benchmark Results. } By using the implementations reported in the literature, we have correctly reproduced the benchmark results. For example, the benchmark results in our Tables \ref{per_1} - \ref{per_6} are comparable to Fig.11 and Fig.9 of \cite{Blalock2020}, Table.4 in \cite{liu2018rethinking},  Fig.3 in \cite{chin2020towards}, Fig. 10 in \cite{frankle2020linear}, Table 5 in \cite{dosovitskiy2020image}, respectively.

%Moreover, the performance of the unpruned network is also an important indicator. As an example, the performance of the unpruned network in our Tables \ref{per_3}, \ref{per_4}, \ref{per_5} and \ref{per_6} are comparable to those reported in \cite{zhao2019variational,chin2020towards,renda2020comparing,dosovitskiy2020image}, i.e., compare our Tables \ref{per_3}, \ref{per_4}, \ref{per_5} and \ref{per_6} to the Table 2 in \cite{zhao2019variational}, Fig.3 in \cite{chin2020towards}, Table 3 in \cite{zhao2019variational}, Table 5 in \cite{dosovitskiy2020image}, respectively.

\textbf{(4) Source Code \& Devices:} We use Tesla V100 devices for our experiments and the source code (including random seeds) will be released at the camera-ready stage.

\begin{table}[!t]
	\centering
	\renewcommand{\arraystretch}{1.0}
	\setlength\tabcolsep{6.5pt}
	\begin{tabular}{|c|cccc|}
		\toprule
		\toprule
		\multicolumn{5}{|c|}{Original Top-1 Test Accuracy: 74.6\% ($\lambda$ = 100\%)} \\ \toprule
		{\footnotesize $\lambda$} & {\footnotesize 32.8\%} & {\footnotesize 26.2\%} & {\footnotesize 13.4\%} & {\footnotesize 5.72\%} \\ \toprule
		LAMP  & 71.5$\pm${\scriptsize 0.7} & 69.6$\pm${\scriptsize 0.8} & 65.8$\pm${\scriptsize 0.9} & 61.2$\pm${\scriptsize 1.4} \\
		AP-Lite & 71.9$\pm${\scriptsize 0.8} & 70.3$\pm${\scriptsize 0.7} & 66.6$\pm${\scriptsize 0.7} & 62.2$\pm${\scriptsize 1.2} \\
		AP-Pro & \textbf{72.2$\pm${\scriptsize 0.7}}& \textbf{71.1$\pm${\scriptsize 0.7}} & \textbf{68.8$\pm${\scriptsize 0.9}} & \textbf{63.5$\pm${\scriptsize 1.5}}  \\ 
		\bottomrule
		\bottomrule
	\end{tabular}
	\vspace{-2.3mm}
	\caption{Performance (top-1 test accuracy $\pm$ standard deviation) of pruning DenseNet-40 on CIFAR-100 using Layer-Adaptive Magnitude Pruning (LAMP) with/without AP.}
	\label{per_3}
	\vspace{-0.8mm}
\end{table}

\begin{table}[!t]
	\centering
	\renewcommand{\arraystretch}{1.0}
	\setlength\tabcolsep{6.5pt}
	\begin{tabular}{|c|cccc|}
		\toprule
		\toprule
		\multicolumn{5}{|c|}{Original Top-1 Test Accuracy: 73.7\% ($\lambda$ = 100\%)} \\ \toprule
		{\footnotesize $\lambda$} & {\footnotesize 32.8\%} & {\footnotesize 26.2\%} & {\footnotesize 13.4\%} & {\footnotesize 5.72\%} \\ \toprule
		LAP  & 72.1$\pm${\scriptsize 0.8} & 70.5$\pm${\scriptsize 0.9} & 67.3$\pm${\scriptsize 0.8} & 64.8$\pm${\scriptsize 1.5} \\
		AP-Lite & 72.5$\pm${\scriptsize 0.9} & 70.9$\pm${\scriptsize 0.8} & 68.2$\pm${\scriptsize 1.2} & 66.2$\pm${\scriptsize 1.5} \\
		AP-Pro & \textbf{72.8$\pm${\scriptsize 0.7}} & \textbf{71.4$\pm${\scriptsize 0.8}} & \textbf{69.1$\pm${\scriptsize 0.8}} & \textbf{67.4$\pm${\scriptsize 1.1}}  \\ 
		\bottomrule
		\bottomrule
	\end{tabular}
	\vspace{-2mm}
	\caption{Performance (top-1 test accuracy $\pm$ standard deviation) of pruning MobileNetV2 on CIFAR-100 using Lookahead Pruning (LAP) with and without AP.}
	\label{per_4}
\end{table}

\subsection{Performance Comparison}
\label{sec4.2}

{\bf (1) Performance using SOTA and Classical Pruning Methods. } In Tables \ref{per_3} and \ref{per_4}, we show that AP can work well with SOTA pruning methods (e.g., LAMP, LAP). In Table \ref{per_3}, we show the performance of AP using LAMP via DenseNet-40 on CIFAR-100. We observe that AP-Lite improves the performance of LAMP by 1.2\% at $\lambda$ = 13.4\% and the improvement increases to 1.6\% at $\lambda$ = 5.7\%. Note that AP-Lite does not increase the algorithm complexity of existing methods. For AP-Pro, it causes a larger improvement of 4.6\% and 3.8\% at $\lambda$ = 13.4\% and $\lambda$ = 5.7\%, respectively. Similar performance trends can be observed in Table \ref{per_4}, where we show the performance of AP using LAP via MobileNetV2 on CIFAR-100. Furthermore, similar performance improvements can be observed using classical pruning methods (Global Magnitude/Gradient) via ResNet-20/VGG-19 on CIFAR-10 as well (see Tables \ref{per_1} - \ref{per_2}). 

{\bf (2) Performance on ImageNet. } In Table \ref{per_5}, we show the performance of AP using Iterative Magnitude Pruning (IMP, i.e., the lottery ticket hypothesis pruning method) via ResNet-50 on ImageNet (i.e., the ILSVRC version) which contains over 1.2 million images from 1000 different classes. We observe that AP-Lite improves the performance of IMP by 1.5\% at $\lambda = 5.7\%$. For AP-Pro, it improves the performance of IMP by 2.8\% at $\lambda = 5.7\%$.

{\bf (3) Performance on SOTA networks (Vision Transformer).} Several recent works \cite{liu2021swin,yuan2021volo,chen2021crossvit} demonstrated that transformer based networks tend to provide excellent performance in computer vision tasks (e.g., classification). We now examine the performance of AP using Vision Transformer (i.e., ViT-B16 with a resolution of 384). We note that the ViT-B16 uses Gaussian Error Linear Units (GELU, GELU(x) = x$\Phi(x)$, where $\Phi(x)$ is the standard Gaussian cumulative distribution function) as the activation function. Similar to ReLU which blocks the negative pre-activation output, GELU heavily regularizes the negative pre-activation output by multiplying an extremely small value of $\Phi(x)$, suggesting that AP could be helpful with pruning GELU based models as well.

We repeat the same experiment setup as above and evaluate the performance of AP using ViT-B16 in Table \ref{per_6}. We observe that AP-Lite helps to improve the performance of IMP by 1.8\% at $\lambda = 5.7\%$. For AP-Pro, it improves the performance of IMP by 3.3\% at $\lambda = 5.7\%$.

\begin{table}[!t]
	\centering
	\renewcommand{\arraystretch}{1.0}
	\setlength\tabcolsep{6.5pt}
	\begin{tabular}{|c|cccc|}
		\toprule
		\toprule
		\multicolumn{5}{|c|}{Original Top-1 Test Accuracy: 77.0\% ($\lambda$ = 100\%)} \\ \toprule
		{\footnotesize $\lambda$} & {\footnotesize 32.8\%} & {\footnotesize 26.2\%} & {\footnotesize 13.4\%} & {\footnotesize 5.72\%} \\ \toprule
		IMP  & 76.8$\pm${\scriptsize 0.2} & 76.4$\pm${\scriptsize 0.3} & 75.2$\pm${\scriptsize 0.4} & 71.5$\pm${\scriptsize 0.4} \\
		AP-Lite & 77.2$\pm${\scriptsize 0.3} & 76.9$\pm${\scriptsize 0.4} & 76.1$\pm${\scriptsize 0.3} & 72.6$\pm${\scriptsize 0.5} \\
		AP-Pro & \textbf{77.5$\pm${\scriptsize 0.4}} & \textbf{77.2$\pm${\scriptsize 0.3}} & \textbf{76.8$\pm${\scriptsize 0.6}} & \textbf{73.5$\pm${\scriptsize 0.4}}  \\ 
		\bottomrule
		\bottomrule
	\end{tabular}
	\vspace{-2mm}
	\caption{Performance (top-1 validation accuracy $\pm$ standard deviation) of pruning ResNet-50 on ImageNet using Iterative Magnitude Pruning (IMP) with and without AP.}
	\label{per_5}
\end{table}

\begin{table}[!t]
	\centering
	\renewcommand{\arraystretch}{1.0}
	\setlength\tabcolsep{6.5pt}
	\begin{tabular}{|c|cccc|}
		\toprule
		\toprule
		\multicolumn{5}{|c|}{Original Top-1 Test Accuracy: 98.0\% ($\lambda$ = 100\%)} \\ \toprule
		{\footnotesize $\lambda$} & {\footnotesize 32.8\%} & {\footnotesize 26.2\%} & {\footnotesize 13.4\%} & {\footnotesize 5.72\%} \\ \toprule
		IMP  & 97.3$\pm${\scriptsize 0.6} & 96.8$\pm${\scriptsize 0.7}  & 88.1$\pm${\scriptsize 0.9} & 82.1$\pm${\scriptsize 0.9} \\
		AP-Lite & 98.0$\pm${\scriptsize 0.4} & 97.3$\pm${\scriptsize 0.7} & 89.9$\pm${\scriptsize 0.6} & 83.6$\pm${\scriptsize 0.8} \\
		AP-Pro & \textbf{98.2$\pm${\scriptsize 0.6}} & \textbf{97.6$\pm${\scriptsize 0.5}} & \textbf{91.1$\pm${\scriptsize 0.8}}  & \textbf{84.8$\pm${\scriptsize 1.0}}  \\ 
		\bottomrule
		\bottomrule
	\end{tabular}
	\vspace{-2mm}
	\caption{Performance (top-1 test accuracy $\pm$ standard deviation) of pruning Vision Transformer (ViT-B-16) on CIFAR-10 using IMP with and without AP.}
	\label{per_6}
\end{table}

\subsection{Ablation Study}
\label{sec4.3}

We now conduct an ablation study to evaluate the effectiveness of components in AP. Specifically, we remove one component at a time in AP and observe the impact on the pruning performance. 
We construct several variants of AP as follows.

\begin{enumerate}
	\setlength{\itemsep}{3pt}
	\item {\bf AP-Lite-NO-WR}: Using AP-Lite without the weight rewinding step (i.e., remove step (ii) from AP-Lite in Algorithm \ref{algorithm2}). This aims to evaluate the effect of weight rewinding on the pruning performance.
	
	\item {\bf AP-Lite-SOLO}: Using only AP-Lite without the benchmark pruning method (i.e., in every pruning cycle, pruning weights only based on AP). This aims to evaluate if the pruning metric of AP can be used to evaluate the importance of weights. 
\end{enumerate}

In Table \ref{ab_1}, we conduct experiments of Pruning ResNet-20 on the CIFAR-10 dataset using Global Magnitude. Based on this configuration, we compare the performance of AP-Lite-NO-WR, AP-Lite-SOLO to AP-Lite so as to demonstrate the effectiveness of components in AP. We note that, same as above, we utilize the implementation reported in the literature. Specifically, the hyper-parameters and the learning rate schedule are from \cite{frankle2018lottery}.

{\bf Effect of Weight Rewinding. } In Table \ref{ab_1}, we compare the performance of AP-Lite-NO-WR to AP-Lite while the key difference is that AP-Lite utilizes weight rewinding (see Algorithm \ref{algorithm2}) and AP-Lite-NO-WR does not. We find that the performance of AP-Lite at $\lambda$ = 51.2\% is 2.4\% higher than AP-Lite-NO-WR. It suggests the crucial role of weight rewinding in improving the performance. Similar performance trends can be observed for other values of $\lambda$. 

%and for Table \ref{ab_2}, where we prune VGG-19 on CIFAR-10 (see Appendix).

{\bf When AP Works Solely. } The pruning metric of AP (see algorithm \ref{algorithm1}) aims to reduce dynamic DNR by pruning. We compare the performance of AP-Lite-SOLO to AP-Lite to evaluate if the pruning metric of AP can be used solely, without working with other pruning methods. In Table \ref{ab_1}, we observe that AP-Lite-SOLO performs much worse than AP-Lite. For example, at $\lambda = 51.2\%$, the performance of AP-Lite-SOLO is 87.1, which is 2.8\% lower than AP-Lite. Similar performance trends can be observed in Table \ref{ab_2} (see {\bf Appendix}), where we prune VGG-19 using CIFAR-10. It suggests that the pruning metric of AP is not suitable to evaluate the importance of weights. The effect of AP's metric on reducing dynamic DNR and its pruning rate $q$ on pruning performance are discussed in the Reflections below.

%is discussed in Reflections. Furthermore, AP reduces dynamic DNR by pruning based on a pruning rate of $q$. The effect of $q$ on performance is also discussed in Reflections.

\begin{table}[!t]
	\centering
	\setlength\tabcolsep{9.0pt}
	\begin{tabular}{|c|ccc|}
		\toprule
		\toprule
		{\small $\lambda$}  & 51.2\% &  40.9\% & 32.8\% \\ \toprule
		{\small AP-Lite} & \textbf{89.6$\pm${\scriptsize 0.5}} & \textbf{88.9$\pm${\scriptsize 0.6}} & \textbf{88.5$\pm${\scriptsize 0.8}} \\ \toprule
		{\small AP-Lite-SOLO} & 87.1$\pm${\scriptsize 0.7} & 86.3$\pm${\scriptsize 0.9} & 85.2$\pm${\scriptsize 1.1} \\
		{\small AP-Lite-NO-WR} & 87.5$\pm${\scriptsize 0.5} & 86.8$\pm${\scriptsize 0.8} & 85.9$\pm${\scriptsize 0.9}\\
		\bottomrule
		\bottomrule
	\end{tabular}
	\vspace{-2.3mm}
	\caption{Ablation Study: Performance Comparison (top-1 test accuracy $\pm$ standard deviation) between AP-Lite and AP-SOLO, AP-Lite-NO-WR on pruning ResNet-20 on CIFAR-10 via Global Magnitude. }
	\label{ab_1}
\end{table}

\section{Reflections}

\label{sec5}
The proposed AP aims to improve existing pruning methods by reducing dynamic DNR. The extensive experiments on popular/SOTA networks and large-scale datasets demonstrate that AP works well with various pruning methods, significantly improving the performance by 3\% - 4\%. We now conclude the paper by presenting some relevant points.

%The proposed AP aims to improve existing pruning methods by reducing dynamic DNR. The extensive experiments on various popular/SOTA networks (e.g., ResNet, Vision Transformer) and large-scale datasets (e.g., ImageNet) demonstrate that AP works well with classical and SOTA pruning methods, significantly improving the performance of existing method by 3 - 8\%. In this section,

%Specifically, AP-Lite improves the performance of existing methods by 2\% - 5\% without increasing the algorithm complexity of pruning. As for AP-Pro, it introduces an additional retraining step, but leads to a more significant improvement of 4\% - 8\%, which could be more preferred in some performance oriented tasks.

{\bf (1) Pruning Rate of AP, $q$. } AP removes $q$\% of remaining parameters in every pruning cycle, so as to reduce dynamic DNR. The value of $q$ is usually much smaller than the pruning rate of the pruning method it works with. Adjusting the value of $q$ is a trade-off between pruning less important weights and reducing dynamic DNR. A large $q$ indicates preferential reduction of dynamic DNR, while a small $q$ means preferential removal of less important weights. We conduct experiments to evaluate the effect of $q$ on the performance and results suggest that a smaller value of $q$ could lead to good performance (see {\bf Appendix} for more details).

{\bf (2) Dynamic DNR with AP. }
We also examine the effect of AP in reducing dynamic DNR. The experimental results on ResNet-20/VGG-19 suggest that AP works as expected and significantly reduce the dynamic DNR. We refer interested reader to {\bf Appendix} for more details.

%We repeat the same experiments in Section \ref{sec3.2} and compare the dynamic DNR with and without using AP-Lite in Tables \ref{dynamic_2} \& \ref{dynamic_1}. In Table \ref{dynamic_2}, we observe that the dynamic DNR is reduced from 9.8\% to 9.1\% at $\lambda = 80\% $ after applying AP-Lite with a pruning rate of 2\%. As $\lambda$ decreases, AP-Lite also works well and reduces the dynamic DNR from 5.1\% to 4.4\% at $\lambda = 20.9\%$. Similar performance trends can be observed in Table \ref{dynamic_1} (see Appendix) as well. This suggests that AP works as expected and explains why AP is able to improve the pruning performance of existing pruning methods. 

{\bf (3) Reducing Static DNR. } The Theorem \ref{thm:1} shows that, in addition to dynamic DNR, reducing static DNR also can improve the upper bound of $I(X;T)$. In fact, reducing static DNR has been incorporated directly or indirectly into the existing pruning methods. As an example, LAMP (i.e., one SOTA pruning method used in performance evaluation, see Table \ref{per_3}) takes the number of unpruned weights of neurons/layers into account and avoids pruning weights from neurons/filters with less number of unpruned weights. This prevents neurons from being statically dead. Differ from existing methods, AP is the first method targeting the dynamic DNR. Hence, as a method that works in tandem with existing pruning methods, AP improves existing pruning methods by filling in the gap in reducing dynamic DNR, leading to much better pruning performance.

\balance
\bibliography{anonymous-submission-latex-2023}

\clearpage
\newpage
\appendix
\onecolumn

\section{Proofs of Theoretical Results}
In this section, we provide the proofs for theoretical results (Theorem 1 and Corollary 1) of the main paper.

\subsection{Proof of Theorem 1}

\setcounter{theorem}{0}
\begin{theorem}
	We are given the computational flow $X\xrightarrow{}T\xrightarrow{}Y$, where $T$ represents the features at some arbitrary depth within a network, represented with finite precision (e.g. float32 or float64). We only consider the subset of network configurations for which (a) the activations in $T$ are less than a threshold $\tau$ and (b) the zero-activation probability of each neuron in $T$ is upper bounded by some $p_S<1$. Let $dim(T)$ represent the dimensionality of $T$, i.e., the number of neurons at that depth. We then have,
	\begin{equation}
		I(X;T)\leq C \times dim(T) \times \left(1-S_{{\scaleto{DNR}{3pt}}}(T) - D_{{\scaleto{DNR}{3pt}}}(T)\left(1 - \frac{1}{C}\log{\frac{1-S_{{\scaleto{DNR}{3pt}}}(T)}{D_{{\scaleto{DNR}{3pt}}}(T)}} \right)\right), 
	\end{equation}
	for a finite constant $C$ that only depends on the network architecture, $\tau$ and $p_S$.
\end{theorem}
\begin{proof}
	First, note that due to finite precision $T$ is a discrete variable, and thus $I(X;T)=H(T)$, as $T=f(X)$ is a deterministic function of $X$, where $f$ denotes the function within the network that maps $X$ to $T$. Next, let us only consider the nodes in $T$ which are not statically dead; i.e. they do not form a part of the $S_{{\scaleto{DNR}{3pt}}}(T)$. Let us denote them as active nodes. Note that there will be $k=dim(T)\times (1-S_{{\scaleto{DNR}{3pt}}}(T))$ active nodes in this case. 
	
	For these $k$ nodes, let $p_1,p_2,...,p_k$ denote the probability that each node will be zero-valued, when $X$ is drawn infinitely over the entire distribution $P$. Let us also denote $D'_{\scaleto{DNR}{3pt}}(T)=\frac{D_{{\scaleto{DNR}{3pt}}}(T)}{1-S_{{\scaleto{DNR}{3pt}}}(T)}$ as the cardinality adjusted dynamic DNR rate of the pruned network. Note that $\mathbb{E}[p_i]=D'_{\scaleto{DNR}{3pt}}(T)$. Let us represent these $k$ nodes by $T_1,T_2,..,T_k$ for what follows. Note that like $T$, each $T_i$ will be discrete valued. We can thus write 
	\begin{equation}
		H(T)\leq \sum H(T_i) 
	\end{equation}
\end{proof}
As all activations are less than $\tau$, if the precision of representation is $\alpha$, we will have a maximum $N = \frac{\tau}{\alpha}$ number of possible outcomes for each $T_i$. Let $\phi^i_0,\phi^i_1,...,\phi^i_{N-1}$ thus represent the probabilities of $T_i$ being each possible discrete outcome. We have that $\sum_j \phi^i_j = 1$. Note that $\phi^i_0=p_i$. 

Now, we can write 
\begin{equation}
	H(T_i) = p_i\log{\frac{1}{p_i}} + (1-p_i)\sum_{N-1\geq j\geq1} \frac{\phi^i_j}{(1-p_i)} \log{\frac{1}{\phi^i_j}}
\end{equation}

Here let us consider the quantity  $\sum_{N-1\geq j\geq1} \frac{\phi^i_j}{(1-p_i)} \log{\frac{1}{\phi^i_j}}$. Let $C$ be the maximum possible value this quantity can take, across all network weight configurations that obey the constraints provided in the Theorem. Note that $C$ will only depend on the network architecture, and the parameters $\tau$ and $p_S$. We will now demonstrate that $C$ is finite, and provide an upper bound for the same. 

Given $p_i=\phi^i_0$, we have that $\sum_{j\geq1}\phi^i_j=1-p_i$. Thus, under this constant summation constraint, the quantity  $\sum_{N-1\geq j\geq1} \frac{\phi^i_j}{(1-p_i)} \log{\frac{1}{\phi^i_j}}$ will only be maximized when $\phi^i_1=\phi^i_2=...\phi^i_{N-1}=\frac{1-p_i}{N-1}$. Thus, we have 
\begin{align}
	C\leq& \sum_{N-1\geq j\geq1} \frac{1}{N-1} \log{\frac{N-1}{1-p_i}} \\
	=&\log{\frac{N-1}{1-p_i}}\leq\log{\frac{N-1}{1-p_S}} 
\end{align}

This shows that $C$ is finite, and depends on the network architecture, $p_S$ and $\tau$ (which affects $N$). Lastly, we have

\begin{align}
	\sum_{i=1}^{k} H(T_i) \leq& \sum_{i=1}^{k} \left (p_i\log{\frac{1}{p_i}} +  (1-p_i)C\right) \\
	=&\sum_{i=1}^{k}p_i\log{\frac{1}{p_i}} + k\times C \times(1-D'_{{\scaleto{DNR}{3pt}}}(T))\\
	\leq& k\times D'_{{\scaleto{DNR}{3pt}}}(T)\log{\frac{1}{D'_{{\scaleto{DNR}{3pt}}}(T)}}+ k\times C\times(1-D'_{{\scaleto{DNR}{3pt}}}(T)) \\
	=& C\times dim(T)\left(1-S_{{\scaleto{DNR}{3pt}}}(T)\right)\left(1-D'_{{\scaleto{DNR}{3pt}}}(T)\left(1-\frac{1}{C}\log{\frac{1}{D'_{{\scaleto{DNR}{3pt}}}(T)}}\right)\right),
\end{align}
where the last step follows from the definition of $k$. Replacing $D'_{{\scaleto{DNR}{3pt}}}(T)=\frac{D_{{\scaleto{DNR}{3pt}}}(T)}{1-S_{{\scaleto{DNR}{3pt}}}(T)}$ yields the intended result.

\subsection{Proof of Corollary 1}
\setcounter{corollary}{0}
\begin{corollary}
	The upper bound for $I(X;T)$ in Theorem \ref{thm:1} decreases in response to increase of both $D_{{\scaleto{DNR}{3pt}}}(T)$ and $S_{{\scaleto{DNR}{3pt}}}(T)$. 
\end{corollary}
\begin{proof}
	It is trivial to see that increasing $S_{{\scaleto{DNR}{3pt}}}(T)$ can only decrease the upper bound. Let us denote
	\begin{equation}
		Z = C\times dim(T)\left(1-S_{{\scaleto{DNR}{3pt}}}(T)\right)\left(1-D'_{{\scaleto{DNR}{3pt}}}(T)\left(1-\frac{1}{C}\log{\frac{1}{D'_{{\scaleto{DNR}{3pt}}}(T)}}\right)\right).
	\end{equation}
	
	For simplicity of notation, let $\beta=C\times dim(T)\times(1-S_{{\scaleto{DNR}{3pt}}}(T))$. 
	For investigating how the upper bound of $I(X;T)$ (denoted as $Z$) changes with $D_{{\scaleto{DNR}{3pt}}}(T)$, we first compute the derivative of $Z$ w.r.t $D'_{{\scaleto{DNR}{3pt}}}(T)$ which yields the following expression. 
	
	\begin{align}
		\frac{dZ}{d(D'_{{\scaleto{DNR}{3pt}}}(T))} =&\beta \left(-\left(1 - \frac{1}{C}\log{\frac{1}{D'_{{\scaleto{DNR}{3pt}}}(T)}} \right)-D'_{{\scaleto{DNR}{3pt}}}(T)\left(\frac{1}{C\times D'_{{\scaleto{DNR}{3pt}}}(T)}\right)\right)\\
		=&\beta\left(-1-\frac{1}{C} + \frac{1}{C}\log{\frac{1}{D'_{{\scaleto{DNR}{3pt}}}(T)}}\right)
	\end{align}
\end{proof}

For $\frac{dZ}{d(D'_{DNR}(T))}$ to be less than or equal to $0$, we must have 
\begin{equation}
	C\geq \log{\frac{1}{D'_{{\scaleto{DNR}{3pt}}}(T)}} - 1.
\end{equation}

In what follows, we will show that $C\geq \log{\frac{1}{D'_{DNR}(T)}}$ itself. Please refer to the proof of Theorem 1 for the definitions. 

For each node output represented in $T_1,T_2,..,T_k$, we add a negative bias of $-\frac{1}{\alpha}$, and flip the sign of all the weights pointing to each of these nodes. We note that performing this change would yield that $\sum_{j=1}^{N-1}\phi^i_j=p_i$. We will then have that 
\begin{equation} \label{temp:eq}
	\sum_{N-1\geq j\geq1} \frac{\phi^i_j}{(p_i)} \log{\frac{1}{\phi^i_j}}\geq \log{\frac{1}{p_i}},
\end{equation}
where the lower bound is achieved by only filling the remaining probability $p_i$ onto a single bin. Now, as $C$ represents the maximum possible value that the quantity bounded in \eqref{temp:eq} can take, we naturally have that $C\geq \log{\frac{1}{p_i}}$ as well, for all $i$. As $\mathbb{E}[p_i]=D'_{{\scaleto{DNR}{3pt}}}(T)$, it also then follows that $C\geq \log{\frac{1}{D'_{{\scaleto{DNR}{3pt}}}(T)}}$. This proves our intended result, and yields that $\frac{dZ}{d(D'_{{\scaleto{DNR}{3pt}}}(T))}\leq0$. Therefore, given that 
$D'_{\scaleto{DNR}{3pt}}(T)=\frac{D_{{\scaleto{DNR}{3pt}}}(T)}{1-S_{{\scaleto{DNR}{3pt}}}(T)}$, a larger value of both $D_{{\scaleto{DNR}{3pt}}}$ and $S_{{\scaleto{DNR}{3pt}}}$ will lead to a larger value of $D'_{\scaleto{DNR}{3pt}}(T)$, decreasing $Z$ (i.e., the upper bound of $I(X;Z)$) and increasing the risk of {\it over-compression}. 

\newpage
\section{Supplementary Experimental Results}
In the Appendix, we show some additional experimental results. Specifically,

\begin{enumerate}
	
	\item In Section \ref{A1}, we repeat the DNR experiment using VGG-19 on the CIFAR-10 dataset with the Global Gradient pruning method. 
	
	\item In Section \ref{A2_new}, we present the implementation details used in the Section of Performance Evaluation (i.e., Section \ref{sec4}) and demonstrate the performance of AP for more values of $\lambda$.
	
	\item In Section \ref{A3_new}, we show the ablation study results using VGG-19 on CIFAR-10 with the Global Gradient pruning method. 
	
	\item In Section \ref{A2}, we evaluate the effect of pruning rate, $q$, on the pruning performance using ResNet-20 and VGG-19 on the CIFAR-10 dataset.
	
	\item In Section \ref{A3}, we evaluate the effect of AP on reducing the dynamic DNR on the pruning performance using ResNet-20 and VGG-19 on the CIFAR-10 dataset.
	
\end{enumerate}

\subsection{The dynamic DNR experiment}
\label{A1}

In this section, we repeat the DNR experiment in the Section of Activating-while-Pruning (i.e., Section \ref{sec3}) using VGG-19 on CIFAR-10 with global gradient pruning method. The hyper-parameters and the LR schedule used are from \cite{frankle2018lottery}. 
As shown in Fig. \ref{fig_2}, we observe the performance trend largely mirrors those reported in Fig. \ref{fig_1}. The dynamic DNR tends to decrease as the network is iteratively pruned (shown in Fig. \ref{fig_2} (left)), and during optimization, the network aims to reduce the dynamic DNR so as to preserve the learning ability of the pruned network.

\vspace{-2mm}\begin{figure}[!ht]
	\hspace{0mm}\begin{minipage}{0.5\textwidth}
		\begin{center}
		\includegraphics[width=0.95\linewidth]{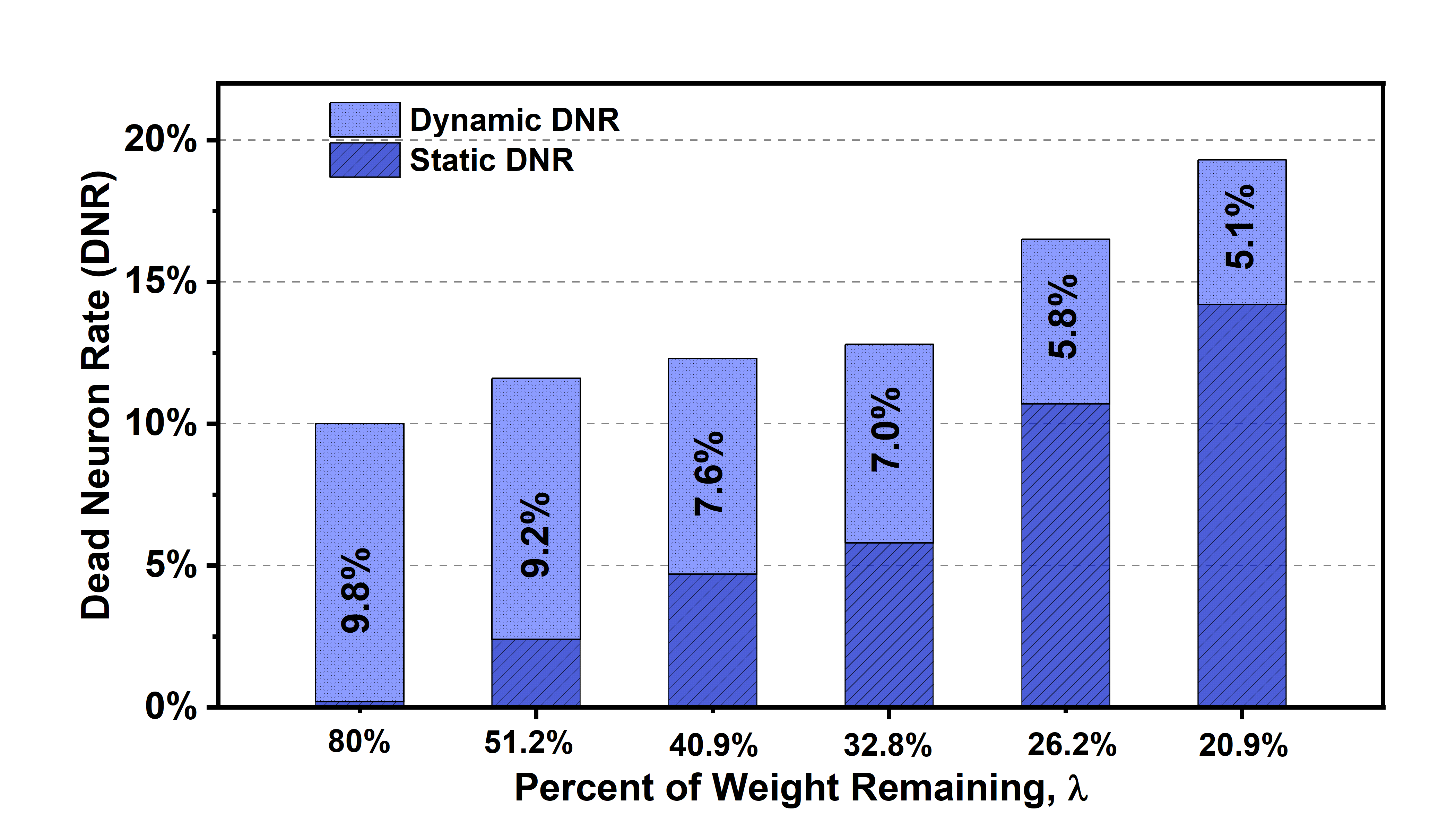}
		\end{center}
	\end{minipage}%
	\hspace{0mm}\begin{minipage}{0.5\textwidth}
		\begin{center}
		\includegraphics[width=0.95\linewidth]{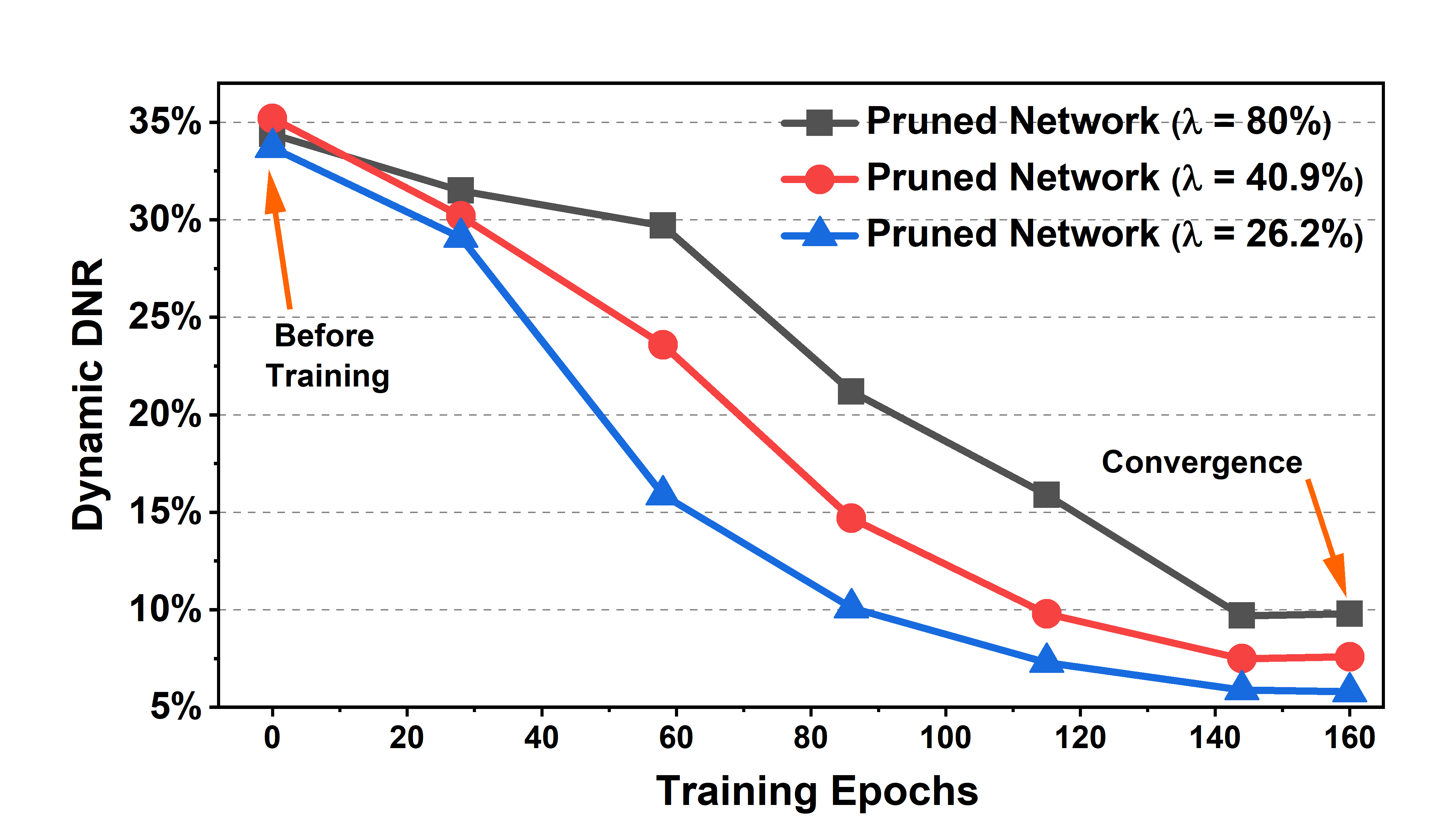}
		\end{center}
	\end{minipage}
	\vspace{-4mm}
	\caption{Dynamic and static Dead Neuron Rate (DNR) when iterative pruning VGG-19 on CIFAR-10 using Global Gradient. Left: dynamic and static DNR when the network converges; Right: dynamic DNR during optimization.} 
	\label{fig_2}
\end{figure}

\subsection{Implementation Details and Performance of AP for More Values of $\lambda$}
\label{A2_new}

In this section, we present the implementation details used in the Section of Performance Evaluation (i.e., Section \ref{sec4}) and demonstrate the performance of AP for more values of $\lambda$.

{\bf (1) Implementation Details. } We use standard implementations reported in the literature. Specifically, the implementation for Tables \ref{per_1_full} - \ref{per_2_full} is from \cite{frankle2018lottery}. The implementation for Table \ref{per_3_full} - \ref{per_6_full} are from \cite{zhao2019variational},  \cite{chin2020towards}, \cite{renda2020comparing} and \cite{dosovitskiy2020image}, respectively. The implementation details can also be found on the top row of each table (from Table \ref{per_1_full} to Table \ref{per_6_full}). Furthermore, for the IMP method examined in this work, we rewind the unpruned weights to their values during training (e.g., epoch 6), in order to obtain a more stable subnetwork \cite{frankle2019stabilizing}.
To better work with IMP, the weight rewinding step in the proposed AP also rewinds the unpruned weights to their values during training (i.e., the same epoch as IMP).

{\bf (2) More Performance of AP. } In Tables \ref{per_1_full}- \ref{per_6_full}, we also show performance of AP for more values of $\lambda$. We observe that for other values of $\lambda$, AP-Lite and AP-Pro also help to achieve higher performance. For example, in Table \ref{per_2_full}, AP-Lite improves the performance of the Global Gradient (at $\lambda$ = 8.59\%) from 84.5\% to 86.1\% while AP-Pro improves the performance to 87.0\%.

\iffalse
We now show the performance of AP using classical pruning methods (Global Magnitude/Gradient) in Tables \ref{per_1} - \ref{per_2}. We note that the implementations for Tables \ref{per_1} - \ref{per_2} are from \cite{frankle2018lottery} and the performance of classical pruning methods is comparable to those reported in the literature (e.g., compare Table \ref{per_1} and Table \ref{per_2} to Fig.9 and Fig. 11 of \cite{Blalock2020} , respectively).

As for the performance, it largely mirrors those in Table \ref{per_3} - \ref{per_6}. We observe that as the percent of weights remaining $\lambda$ decreases, the improvement of AP becomes larger. For example, in Table \ref{per_1}, the performance of AP-Lite at $\lambda = 26.2\%$ is 1.3\% higher than the benchmark result obtained using the global magnitude. The improvement increases to 2.6\% at $\lambda = 5.7\%$. Note that AP-Lite does not increase the algorithm complexity of existing methods. As for the performance of AP-Pro, it works as expected. In Table \ref{per_1}, AP-Pro leads to a more significant improvement of 2.0\% and 4.1\% at $\lambda = 26.2\%$ and $\lambda = 5.7\%$, respectively. 

Similar performance trends can also be observed in Table \ref{per_2}, where we prune VGG-19 on CIFAR-10 using the global gradient pruning method.
\fi

\begin{table}[!ht]
	\centering
	\setlength\tabcolsep{20.0pt}
	\begin{tabular}{|c|c|cc|}
		\toprule
		\toprule
		\multicolumn{4}{|c|}{(i) Params: 270K;  (ii) Train Steps: 100 Epochs; (iii) Batch size: 128; } \\
		\multicolumn{4}{|c|}{(iv) LR Schedule: warmup to 0.03 at 55 epochs, 10X drop at 55, 70 epochs.} \\ \toprule
		{\footnotesize Percent of Weights Remaining} & {\footnotesize Global Magnitude} & {\footnotesize AP-Lite} & {\footnotesize AP-Pro} \\ \toprule
		{\small  $\lambda$ = 100.0\%} &  91.7$\pm${\scriptsize 0.2} & 91.7$\pm${\scriptsize 0.2} & 91.7$\pm${\scriptsize 0.5} \\
		%{\small  $\lambda$ = 80.0\%} &  89.9$\pm${\scriptsize 0.5} & 90.1$\pm${\scriptsize 0.4} & 90.1$\pm${\scriptsize 0.4} \\
		{\small  $\lambda$ = 64.0\%} & 91.5$\pm${\scriptsize 0.3} & 91.7$\pm${\scriptsize 0.2} & \textbf{91.8$\pm${\scriptsize 0.3}}  \\
		%{\small  $\lambda$ = 51.2\%} &  88.8$\pm${\scriptsize 0.4} & 89.6$\pm${\scriptsize 0.5} & 89.9$\pm${\scriptsize 0.3} \\
		{\small  $\lambda$ = 40.9\%} &  90.8$\pm${\scriptsize 0.5} & 91.0$\pm${\scriptsize 0.6} & \textbf{91.4$\pm${\scriptsize 0.4}} \\
		{\small  $\lambda$ = 32.8\%} &  90.3$\pm${\scriptsize 0.4} & 90.4$\pm${\scriptsize 0.7} & \textbf{90.7$\pm${\scriptsize 0.6}} \\
		{\small  $\lambda$ = 26.2\%} &  89.8$\pm${\scriptsize 0.6} & 90.2$\pm${\scriptsize 0.8} & \textbf{90.4$\pm${\scriptsize 0.7}} \\
		{\small  $\lambda$ = 13.4\%} &  88.2$\pm${\scriptsize 0.7} & 88.7$\pm${\scriptsize 0.7} & \textbf{89.3$\pm${\scriptsize 0.8}} \\
		{\small  $\lambda$ = 8.59\%} &  85.9$\pm${\scriptsize 0.9} & 86.8$\pm${\scriptsize 0.9} & \textbf{87.3$\pm${\scriptsize 0.8}} \\
		{\small  $\lambda$ = 5.72\%} &  81.2$\pm${\scriptsize 1.1} & 82.4$\pm${\scriptsize 0.8} & \textbf{84.1$\pm${\scriptsize 1.1}} \\
		\bottomrule
		\bottomrule
	\end{tabular}
	\vspace{-3mm}
	\caption{Performance (top-1 test accuracy $\pm$ standard deviation) of pruning ResNet-20 on CIFAR-10 using Global Magnitude with and without AP. The hyper-parameters and the LR schedule are from \cite{frankle2018lottery}.}
	\label{per_1_full}
	\vspace{4mm}
\end{table}

\begin{table}[!ht]
	\centering
	\setlength\tabcolsep{20.7pt}
	%\setlength\extrarowheight{-0.2pt}
	%\vspace*{-\baselineskip}
	\begin{tabular}{|c|c|cc|}
		\toprule
		\toprule
		\multicolumn{4}{|c|}{(i) Params: 139M; (ii) Train Steps: 160 epochs; (iii) Batch size: 64; } \\
		\multicolumn{4}{|c|}{(iv) LR Schedule: warmup to 0.1 at 15 epochs, 10X drop at 85, 125 epochs.} \\ \toprule
		{\footnotesize Percent of Weights Remaining} & {\footnotesize Global Gradient} & {\footnotesize AP-Lite} & {\footnotesize AP-Pro} \\ \toprule
		{\small  $\lambda$ = 100.0\%} &  92.2$\pm${\scriptsize 0.3} & 92.2$\pm${\scriptsize 0.3} & 92.2$\pm${\scriptsize 0.3} \\
		{\small  $\lambda$ = 64.0\%} &  91.3$\pm${\scriptsize 0.2} & 91.5$\pm${\scriptsize 0.3} & \textbf{91.9$\pm${\scriptsize 0.3}}  \\
		{\small  $\lambda$ = 40.9\%} &  90.6$\pm${\scriptsize 0.4} & 90.8$\pm${\scriptsize 0.5} & \textbf{91.1$\pm${\scriptsize 0.7}} \\
		{\small  $\lambda$ = 32.8\%} &  90.2$\pm${\scriptsize 0.5} & 90.5$\pm${\scriptsize 0.8} & \textbf{90.8$\pm${\scriptsize 0.6}} \\
		{\small  $\lambda$ = 26.2\%} &  89.8$\pm${\scriptsize 0.8} & 90.3$\pm${\scriptsize 0.7} & \textbf{90.7$\pm${\scriptsize 0.9}} \\
		{\small  $\lambda$ = 13.4\%} &  89.2$\pm${\scriptsize 0.8} & 89.7$\pm${\scriptsize 0.9} & \textbf{90.4$\pm${\scriptsize 0.8}} \\
		{\small  $\lambda$ = 8.59\%} &  84.5$\pm${\scriptsize 0.9} & 86.1$\pm${\scriptsize 1.0} & \textbf{87.0$\pm${\scriptsize 0.7}} \\
		{\small  $\lambda$ = 5.72\%} &  76.9$\pm${\scriptsize 1.1} & 78.4$\pm${\scriptsize 1.4} & \textbf{79.2$\pm${\scriptsize 1.3}} \\
		\bottomrule
		\bottomrule
	\end{tabular}
    \vspace{-3mm}
	\caption{Performance (top-1 test accuracy $\pm$ standard deviation) of pruning VGG-19 on CIFAR-10 using Global Gradient with and without AP.  The hyper-parameters and the LR schedule are from \cite{frankle2018lottery}.}
	\label{per_2_full}
	\vspace{4mm}
\end{table}

\begin{table}[!ht]
	\centering
	\renewcommand{\arraystretch}{1.0}
	\setlength\tabcolsep{23.7pt}
	\begin{tabular}{|c|c|cc|}
		\toprule
		\toprule
		\multicolumn{4}{|c|}{(i) Params: 1.1M;  (ii) Train Steps: 300 epochs; (iii) Batch size: 256; } \\
		\multicolumn{4}{|c|}{(iv) LR Schedule: warmup to 0.1 at 150 epochs, 10X drop at 150, 240 epochs.} \\ \toprule
		{\footnotesize Percent of Weights Remaining} & {\footnotesize LAMP} & {\footnotesize AP-Lite} & {\footnotesize AP-Pro} \\ \toprule
		{\small  $\lambda$ = 100.0\%} &  74.6$\pm${\scriptsize 0.5} & 74.6$\pm${\scriptsize 0.5} & 74.6$\pm${\scriptsize 0.5} \\
		%{\small  $\lambda$ = 80.0\%} &  70.7$\pm${\scriptsize 0.4} & 70.9$\pm${\scriptsize 0.4} & 70.9$\pm${\scriptsize 0.4} \\
		{\small  $\lambda$ = 64.0\%} &  73.4$\pm${\scriptsize 0.6} & 73.7$\pm${\scriptsize 0.5} & \textbf{74.2$\pm${\scriptsize 0.6}}  \\
		%{\small  $\lambda$ = 51.2\%} &  69.9$\pm${\scriptsize 0.5} & 70.5$\pm${\scriptsize 0.6} & 70.9$\pm${\scriptsize 0.8} \\
		%{\small  $\lambda$ = 40.9\%} &  69.2$\pm${\scriptsize 0.8} & 70.2$\pm${\scriptsize 0.7} & \textbf{70.6$\pm${\scriptsize 0.5}} \\
		{\small  $\lambda$ = 32.8\%} &  71.5$\pm${\scriptsize 0.7} & 71.9$\pm${\scriptsize 0.8} & \textbf{72.2$\pm${\scriptsize 0.7}} \\
		{\small  $\lambda$ = 26.2\%} &  69.6$\pm${\scriptsize 0.8} & 70.3$\pm${\scriptsize 0.7} & \textbf{71.1$\pm${\scriptsize 0.7}} \\
		{\small  $\lambda$ = 13.4\%} &  65.8$\pm${\scriptsize 0.9} & 66.6$\pm${\scriptsize 0.7} & \textbf{68.8$\pm${\scriptsize 0.9}} \\
		%{\small  $\lambda$ = 8.59\%} &  59.4$\pm${\scriptsize 1.5} & 61.2$\pm${\scriptsize 1.1} & \textbf{63.1$\pm${\scriptsize 1.4}} \\
		{\small  $\lambda$ = 5.72\%} &  61.2$\pm${\scriptsize 1.4} & 62.2$\pm${\scriptsize 1.2} & \textbf{63.5$\pm${\scriptsize 1.5}} \\
		\bottomrule
		\bottomrule
	\end{tabular}
	\vspace{-3mm}
	\caption{Performance (top-1 test accuracy $\pm$ standard deviation) of pruning DenseNet-40 on CIFAR-100 using Layer-Adaptive Magnitude Pruning (LAMP) \cite{lee2020layer} with/without AP. The hyper-parameters and the LR schedule are from \cite{zhao2019variational}.}
	\label{per_3_full}
	\vspace{4mm}
\end{table}

\begin{table}[!ht]
	\centering
	\renewcommand{\arraystretch}{0.9}
	\setlength\tabcolsep{23.7pt}
	\begin{tabular}{|c|c|cc|}
		\toprule
		\toprule
		\multicolumn{4}{|c|}{(i) Params: 2.36M; (ii) Train Steps: 200 epochs; (iii) Batch size: 128;} \\
		\multicolumn{4}{|c|}{(iv) LR Schedule: warmup to 0.1 at 60 epochs, 10X drop at 60, 120, 160 epochs.} \\ \toprule
		{\footnotesize Percent of Weights Remaining} & {\footnotesize LAP} & {\footnotesize AP-Lite} & {\footnotesize AP-Pro} \\ \toprule
		{\small  $\lambda$ = 100.0\%} &  73.7$\pm${\scriptsize 0.4} & 73.7$\pm${\scriptsize 0.4}& 73.7$\pm${\scriptsize 0.4} \\
		%{\small  $\lambda$ = 80.0\%} &  70.6$\pm${\scriptsize 0.3} & 70.9$\pm${\scriptsize 0.4} & 70.9$\pm${\scriptsize 0.4} \\
		{\small  $\lambda$ = 64.0\%} &  72.5$\pm${\scriptsize 0.4} & 72.7$\pm${\scriptsize 0.3} & \textbf{72.9$\pm${\scriptsize 0.5}}  \\
		%{\small  $\lambda$ = 51.2\%} &  69.5$\pm${\scriptsize 0.6} & 70.1$\pm${\scriptsize 0.6} & 70.5$\pm${\scriptsize 0.7} \\
		%{\small  $\lambda$ = 40.9\%} &  68.9$\pm${\scriptsize 0.7} & 69.7$\pm${\scriptsize 0.8} & 70.3$\pm${\scriptsize 0.6} \\
		{\small  $\lambda$ = 32.8\%} &  72.1$\pm${\scriptsize 0.8} & 72.5$\pm${\scriptsize 0.9} & \textbf{72.8$\pm${\scriptsize 0.7}} \\
		{\small  $\lambda$ = 26.2\%} &  70.5$\pm${\scriptsize 0.9} & 70.9$\pm${\scriptsize 0.8} & \textbf{71.4$\pm${\scriptsize 0.8}} \\
		{\small  $\lambda$ = 13.4\%} &  67.3$\pm${\scriptsize 0.8} & 68.2$\pm${\scriptsize 1.2} & \textbf{69.1$\pm${\scriptsize 0.8}} \\
		%{\small  $\lambda$ = 8.59\%} &  59.8$\pm${\scriptsize 1.3} & 62.1$\pm${\scriptsize 1.4} & \textbf{64.4$\pm${\scriptsize 1.5}} \\
		{\small  $\lambda$ = 5.72\%} &  64.8$\pm${\scriptsize 1.5} & 66.2$\pm${\scriptsize 1.5} & \textbf{67.4$\pm${\scriptsize 1.1}} \\
		\bottomrule
		\bottomrule
	\end{tabular}
	\vspace{-3mm}
	\caption{Performance (top-1 test accuracy $\pm$ standard deviation) of pruning MobileNetV2 on CIFAR-100 using  Lookahead Pruning (LAP) \cite{park2020lookahead} with/without AP. The hyper-parameters and the LR schedule are from \cite{chin2020towards}.}
	\label{per_4_full}
\end{table}

\begin{table}[!ht]
	\centering
	\renewcommand{\arraystretch}{0.9}
	\setlength\tabcolsep{23.4pt}
	\begin{tabular}{|c|c|cc|}
		\toprule
		\toprule
		\multicolumn{4}{|c|}{(i) Params: 25.5M; (ii) Train Steps: 90 epochs; (iii) Batch size: 1024;} \\
		\multicolumn{4}{|c|}{(iv) LR Schedule: warmup to 0.4 at 5 epochs, 10X drop at 30, 60, 80 epochs.} \\ \toprule
		{\footnotesize Percent of Weights Remaining} & {\footnotesize IMP} & {\footnotesize AP-Lite} & {\footnotesize AP-Pro} \\ \toprule
		{\small  $\lambda$ = 100.0\%} &  77.0$\pm${\scriptsize 0.1} & 77.0$\pm${\scriptsize 0.1} & 77.0$\pm${\scriptsize 0.1} \\
		%{\small  $\lambda$ = 80.0\%} &  57.1$\pm${\scriptsize 0.2} & 57.5$\pm${\scriptsize 0.3} & 57.5$\pm${\scriptsize 0.3} \\
		{\small  $\lambda$ = 64.0\%} &  77.2$\pm${\scriptsize 0.2} & 77.5$\pm${\scriptsize 0.1} & \textbf{77.7$\pm${\scriptsize 0.1}}  \\
		%{\small  $\lambda$ = 51.2\%} &  56.4$\pm${\scriptsize 0.6} & 56.8$\pm${\scriptsize 0.4} & 57.1$\pm${\scriptsize 0.7} \\
		%{\small  $\lambda$ = 40.9\%} &  72.3$\pm${\scriptsize 0.7} & 72.8$\pm${\scriptsize 0.7} & 73.1$\pm${\scriptsize 0.9} \\
		{\small  $\lambda$ = 32.8\%} &  76.8$\pm${\scriptsize 0.2} & 77.2$\pm${\scriptsize 0.3} & \textbf{77.5$\pm${\scriptsize 0.4}} \\
		{\small  $\lambda$ = 26.2\%} &  76.4$\pm${\scriptsize 0.3} & 76.9$\pm${\scriptsize 0.4} & \textbf{77.2$\pm${\scriptsize 0.3}} \\
		{\small  $\lambda$ = 13.4\%} &  75.2$\pm${\scriptsize 0.4} & 76.1$\pm${\scriptsize 0.3} & \textbf{76.8$\pm${\scriptsize 0.6}} \\
		{\small  $\lambda$ = 8.59\%} &  73.8$\pm${\scriptsize 0.5} & 75.2$\pm${\scriptsize 0.7} & \textbf{75.9$\pm${\scriptsize 0.5}} \\
		{\small  $\lambda$ = 5.72\%} &  71.5$\pm${\scriptsize 0.4} & 72.6$\pm${\scriptsize 0.5} & \textbf{73.5$\pm${\scriptsize 0.4}} \\
		\bottomrule
		\bottomrule
	\end{tabular}
	\vspace{-3mm}
	\caption{Performance (top-1 test accuracy $\pm$ standard deviation) of pruning ResNet-50 on ImageNet using Iterative Magnitude Pruning (IMP) with and without AP \cite{frankle2018lottery}. The hyper-parameters and the LR schedule are from \cite{renda2020comparing}.}
	\label{per_5_full}
\end{table}

\begin{table}[!ht]
	\centering
	\renewcommand{\arraystretch}{1.0}
	\setlength\tabcolsep{23.4pt}
	\begin{tabular}{|c|c|cc|}
		\toprule
		\toprule
		\multicolumn{4}{|c|}{(i) Params: 86M; (ii) Train Steps: 50 epochs; (iii) Batch size: 1024;} \\
		\multicolumn{4}{|c|}{(iv) Optimizer: Adam; (v) LR Schedule: cosine decay from 1e-4.} \\ \toprule
		{\footnotesize Percent of Weights Remaining} & {\footnotesize IMP} & {\footnotesize AP-Lite} & {\footnotesize AP-Pro} \\ \toprule
		{\small  $\lambda$ = 100.0\%} &  98.0$\pm${\scriptsize 0.3} & 98.0$\pm${\scriptsize 0.3} & 98.0$\pm${\scriptsize 0.3} \\
		%{\small  $\lambda$ = 80.0\%} &  57.1$\pm${\scriptsize 0.2} & 57.5$\pm${\scriptsize 0.3} & 57.5$\pm${\scriptsize 0.3} \\
		{\small  $\lambda$ = 64.0\%} &  98.4$\pm${\scriptsize 0.3} & 98.5$\pm${\scriptsize 0.2} & \textbf{98.7$\pm${\scriptsize 0.3}}  \\
		%{\small  $\lambda$ = 51.2\%} &  56.4$\pm${\scriptsize 0.6} & 56.8$\pm${\scriptsize 0.4} & 57.1$\pm${\scriptsize 0.7} \\
		%{\small  $\lambda$ = 40.9\%} &  72.3$\pm${\scriptsize 0.7} & 72.8$\pm${\scriptsize 0.7} & 73.1$\pm${\scriptsize 0.9} \\
		{\small  $\lambda$ = 32.8\%} &  97.3$\pm${\scriptsize 0.6} & 98.0$\pm${\scriptsize 0.4} & \textbf{98.2$\pm${\scriptsize 0.6}} \\
		{\small  $\lambda$ = 26.2\%} &  96.8$\pm${\scriptsize 0.7} & 97.3$\pm${\scriptsize 0.7} & \textbf{97.6$\pm${\scriptsize 0.5}} \\
		{\small  $\lambda$ = 13.4\%} &  88.1$\pm${\scriptsize 0.9} & 89.9$\pm${\scriptsize 0.6} & \textbf{91.1$\pm${\scriptsize 0.8}} \\
		{\small  $\lambda$ = 8.59\%} &  84.4$\pm${\scriptsize 0.8} & 85.5$\pm${\scriptsize 0.8} & \textbf{87.4$\pm${\scriptsize 0.7}} \\
		{\small  $\lambda$ = 5.72\%} &  82.1$\pm${\scriptsize 0.9} & 83.6$\pm${\scriptsize 0.8} & \textbf{84.8$\pm${\scriptsize 1.0}} \\
		\bottomrule
		\bottomrule
	\end{tabular}
	\vspace{-3mm}
	\caption{Performance (top-1 test accuracy $\pm$ standard deviation) of pruning Vision Transformer \cite{dosovitskiy2020image} (ViT-B-16) on CIFAR-10 using IMP with and without AP \cite{frankle2018lottery}. The hyper-parameters and the LR schedule are from \cite{dosovitskiy2020image}.}
	\label{per_6_full}
\end{table}

\clearpage
\newpage
\subsection{Ablation Study Results using VGG-19}
\label{A3_new}

In this section, we repeat the ablation study experiment in the Section of Ablation Study (i.e., Section \ref{sec4.3}) using VGG-19 on CIFAR-10 with the Global Gradient method. We note that the hyper-parameters and the LR schedule used are from \cite{frankle2018lottery}. 

As show in Table \ref{ab_2}, we observe similar performance trends as Table \ref{ab_1}. Specifically, the performance of AP-Lite-NO-WR is much lower as compared to AP-Lite, suggesting the effectiveness of weight rewinding (WR). Furthermore, when AP works solely, the performance (i.e., AP-Lite-SOLO) tends to become much worse than AP-Lite. This agrees with our argument that AP does not work solely as it does not evaluate the importance of weights. Instead, the goal of AP is to work in tandem with existing pruning methods and aims to improve their performance by reducing the dynamic DNR. 

\begin{table}[!ht]
	\centering
	\setlength\tabcolsep{9.0pt}
	\begin{tabular}{|c|cccc|}
		\toprule
		\toprule
		{\small $\lambda$} & 64\% & 51.2\% &  40.9\% & 32.8\% \\ \toprule
		{\small AP-Lite} & \textbf{91.7 $\pm$ 0.3} & \textbf{90.6 $\pm$ 0.5} & \textbf{89.8 $\pm$ 0.9} & \textbf{89.5 $\pm$ 0.8} \\ \toprule
		{\small AP-Lite-SOLO} & 90.5 $\pm$ 0.7 & 89.8 $\pm$ 0.6  & 88.3 $\pm$ 1.1 & 87.2 $\pm$ 1.3 \\
		{\small AP-Lite-NO-WR} & 90.8 $\pm$ 0.4 & 90.1 $\pm$ 0.3 & 88.5 $\pm$ 0.9 & 87.7 $\pm$ 1.0 \\
		\bottomrule
		\bottomrule
	\end{tabular}
	\vspace{-3mm}
	\caption{Ablation Study: Performance Comparison (top-1 test accuracy $\pm$ standard deviation) between AP-Lite and AP-SOLO, AP-Lite-NO-WR on pruning VGG-19 using the CIFAR-10 dataset via Global Gradient. The hyper-parameters and the LR schedule are from \cite{frankle2018lottery}.}
	\label{ab_2}
\end{table}

\subsection{Pruning Rate of AP, $q$}
\label{A2}

%we repeat the experiment in Section \ref{sec5} (i.e., the effect of pruning rate, $q$) using the VGG-19 on CIFAR-10 with the global gradient pruning method. The hyper-parameters and the LR schedule used are from \cite{frankle2018lottery}. 

In this subsection, we repeat the experiments of pruning ResNet-20 on CIFAR-10 using Global Magnitude and AP-Lite. We note that the overall pruning rate is fixed as 20\% and the pruning rate of AP increases from 1\% to 5\%. Correspondingly, the pruning rate of Global Magnitude decreases from 19\% to 15\%. The experimental results are summarized in Table \ref{pr_1}. We observe that as we increase the pruning rate of AP from 2\%, the performance tends to decrease. Similar performance trends can be observed using VGG-19 on CIFAR-10 as well (see Table \ref{pr_2}). As an example, in Table \ref{pr_2}, the top-1 test accuracy is reduced from 88.5 to 87.1 for $\lambda = 26.2\%$ when the value of $q$ increases from 2\% to 5\%. It suggests that the primary goal of pruning still should be pruning less important weights.

The theoretical determination of the optimal value of $q$ is clearly worth deeper thought. Alternatively, $q$ can be thought of as a hyper-parameter and tuned via the validation dataset and let $q$ = 2 could be a good choice as it provides promising results in various experiments.

%As shown in Table \ref{pr_2}, the performance largely mirror those in Table \ref{pr_1}. As we increase the value of $q$, the performance tends to decrease significantly. As an example, the top-1 test accuracy is reduced from 88.5 to 87.1 for $\lambda = 26.2\%$ when the value of $q$ increases from 2\% to 5\%. It suggests that the primary goal of pruning still should be pruning less important weights.

%Specifically, without AP, the pruning rate is $p$ (i.e., $p\%$ of remaining weights are pruned). After using AP, $(p-q)\%$ of remaining weights are pruned according to the benchmark method while $q\%$ of remaining weights are pruned according to AP. Therefore, adjusting the value of $q$ is a trade-off between pruning less important weights and reducing dynamic DNR.

\begin{table}[!ht]
	\centering
	\setlength\tabcolsep{9.0pt}
	\begin{tabular}{|c|cccc|}
		\toprule
		\toprule
		{AP Pruning Rate, $q$} & 1\% & 2\% &  3\% & 5\% \\ \toprule
		{\small AP-Lite ($\lambda$ = 64.0\%)} & 89.8 $\pm$ 0.1 & \textbf{90.0 $\pm$ 0.2} & 89.5 $\pm$ 0.4 & 89.2 $\pm$ 0.6 \\
		
		{\small AP-Lite ($\lambda$ = 40.9\%)} & 88.2 $\pm$ 0.4 & \textbf{88.9 $\pm$ 0.6} & 88.5 $\pm$ 0.7 & 87.3 $\pm$ 0.5 \\
		
		{\small AP-Lite ($\lambda$ = 26.2\%)} & 87.1 $\pm$ 0.5 & \textbf{87.9 $\pm$ 0.8} & 86.7 $\pm$ 0.8 & 86.3 $\pm$ 0.9 \\
		\bottomrule
		\bottomrule
	\end{tabular}
\vspace{-2mm}
	\caption{Performance (top-1 test accuracy $\pm$ standard deviation) of AP-Lite when iterative pruning ResNet-20 on CIFAR-10 with different pruning rate.}
	\label{pr_1}
\end{table}

\begin{table}[!ht]
	\centering
	\setlength\tabcolsep{9.0pt}
	\begin{tabular}{|c|cccc|}
		\toprule
		\toprule
		{AP Pruning Rate, $q$} & 1\% & 2\% &  3\% &  5\% \\ \toprule
		{\small AP-Lite ($\lambda$ = 64.0\%)} & 91.1 $\pm$ 0.2 & \textbf{91.7 $\pm$ 0.3} & 90.2 $\pm$ 0.5 & 89.8 $\pm$ 0.6 \\
		
		{\small AP-Lite ($\lambda$ = 40.9\%)} & 88.7 $\pm$ 0.5 & \textbf{89.8 $\pm$ 0.9} & 89.3 $\pm$ 1.1 & 88.0 $\pm$ 0.9 \\
		
		{\small AP-Lite ($\lambda$ = 26.2\%)} & 87.9 $\pm$ 0.8 & \textbf{88.5 $\pm$ 0.7} & 88.1 $\pm$ 1.3 & 87.1 $\pm$ 1.3 \\
		\bottomrule
		\bottomrule
	\end{tabular}
	\vspace{-3mm}
	\caption{Performance (top-1 test accuracy $\pm$ standard deviation) of AP-Lite when iterative pruning VGG-19 on CIFAR-10 with different pruning rate.}
	\label{pr_2}
	
\end{table}

\newpage
\subsection{Effect of AP on Dynamic DNR}
\label{A3}

In this section, we repeat the same experiments in the Section of Performance Evaluation (i.e., Section \ref{sec4}) and compare the dynamic DNR with and without using AP-Lite in Tables \ref{dynamic_2} \& \ref{dynamic_1}.

In Table \ref{dynamic_2}, we observe that dynamic DNR is reduced from 9.8\% to 9.1\% at $\lambda = 80\% $ after applying AP-Lite with a pruning rate of 2\%. As $\lambda$ decreases, AP-Lite also works well and reduces dynamic DNR from 5.1\% to 4.4\% at $\lambda = 20.9\%$. Similar performance trends can also be observed in Table \ref{dynamic_1}. This suggests that AP works as expected and explains why AP is able to improve the pruning performance of existing pruning methods.

%we repeat the experiment in Section \ref{sec5} (i.e., the effect of AP on dynamic DNR) using ResNet-20 on the CIFAR-10 dataset with the global magnitude pruning method. Note that we utilize the standard implementation details (e.g., optimizer and learning rate schedules) from \cite{frankle2018lottery}.

%In Table \ref{dynamic_1}, we observe similar performance trends as Table \ref{dynamic_2}. Specifically, AP works as expected and significantly reduces the dynamic DNR. As an example, the dynamic DNR is reduced from 6.1\% to 5.7\% at $\lambda = 20.9\% $, with an improvement of 7\%. This further verifies the effect of AP on reducing dynamic DNR using another network and pruning methods.

\begin{table}[!ht]
	\centering
	\setlength\tabcolsep{9.3pt}
	\begin{tabular}{|c|cccccc|}
		\toprule
		\toprule
		{\small $\lambda$} & 80\% & 51.2\% &  40.9\% & 32.8\% & 26.2\% & 20.9\% \\ \toprule
		
		{\small Global Magnitude} & 7.7\% & 7.4\% & 7.0\% & 6.7\% & 6.3\% & 6.1\% \\ \toprule
		
		{\small AP-Lite (2\%)} & \bf 7.3\% & \bf 7.1\% & \bf 6.8\% & \bf 6.3\% & \bf 5.9\% & \bf 5.5\% \\
		\bottomrule
		\bottomrule
	\end{tabular}
	\vspace{-3mm}
	\caption{The dynamic DNR when iteratively pruning ResNet-20 on CIFAR-10 using Global Magnitude and AP-Lite (with a pruning rate of 2\%). }
	\label{dynamic_1}
\end{table}

\begin{table}[!ht]
	\centering
	\setlength\tabcolsep{9.85pt}
	\vspace*{-\baselineskip}
	\begin{tabular}{|c|cccccc|}
		\toprule
		\toprule
		{\small $\lambda$} & 80\% & 51.2\% &  40.9\% & 32.8\% & 26.2\% & 20.9\% \\ \toprule
		
		{\small Global Gradient} & 9.8\% & 9.2\% & 7.6\% & 7.0\% & 5.8\% &  5.1\%  \\ \toprule
		
		{\small AP-Lite (2\%)} & \bf 9.1\% & \bf 8.6\% & \bf 7.0\% & \bf 6.1\% & \bf 4.9\% & \bf 4.4\% \\
		\bottomrule
		\bottomrule
	\end{tabular}
\vspace{-2mm}
	\caption{The dynamic DNR when iteratively pruning VGG-19 on CIFAR-10 using Global Gradient and AP-Lite (with a pruning rate of 2\%). }
	\label{dynamic_2}
\end{table}

\end{document}